\documentclass[a4paper,fleqn]{cas-dc_classicmath}

\usepackage{placeins}

\usepackage[numbers]{natbib}
\usepackage{dsfont}

\usepackage{graphics}
\usepackage{amsmath}
\usepackage{amssymb}
\usepackage{algorithm}
\usepackage{algorithmicx}
\usepackage{algpseudocode}
\usepackage{mathtools}
\usepackage{tikz}
\usepackage{bm}
\usepackage{empheq}
\usepackage[utf8]{inputenc}
\usepackage[most]{tcolorbox}
\usepackage{subfigure}

\newtheorem{dfn}{Definition}
\newtheorem{proposition}{Proposition}
\newtheorem{theorem}{Theorem}
\newtheorem{example}[theorem]{Example}

\newtheorem{proof}{Proof}

\usepackage{float}
\usepackage{microtype}
\usepackage{xcolor}
\usepackage[normalem]{ulem}

\newif\ifpublicpreprint
\publicpreprinttrue   

\ifpublicpreprint
\ExplSyntaxOn
\cs_set:Npn \__first_footerline:
{
	\group_begin:
	\small
	\sffamily
	\ifnum\theblind>0\relax
	\else
	\__short_authors: :~
	\fi
	{ \rmfamily \itshape Preprint ~--~ October~2026 }
	\group_end:
}
\ExplSyntaxOff
\fi

\begin{document}
\let\WriteBookmarks\relax
\renewcommand{\topfraction}{0.95}
\renewcommand{\bottomfraction}{0.90}
\renewcommand{\dbltopfraction}{0.95}
\renewcommand{\textfraction}{0.05}
\renewcommand{\floatpagefraction}{0.70}
\renewcommand{\dblfloatpagefraction}{0.70}
\setcounter{topnumber}{4}
\setcounter{bottomnumber}{4}
\setcounter{totalnumber}{8}
\setcounter{dbltopnumber}{4}

\shorttitle{Contractor Filters}

\shortauthors{Rohou}

\title [mode = title]{Contractor Filters}

\author[1]{Simon Rohou}[orcid=0000-0001-6232-9918]
\cormark[1]

\ead{simon.rohou@ensta.fr}

\ead[url]{https://www.simon-rohou.fr/research/}

\credit{}

\affiliation[1]{organization={ENSTA (Polytechnic Institute of Paris), Lab-STICC, UMR CNRS 6285}, addressline={2 rue François Verny}, city={Brest}, postcode={29200}, country={France}}

\cortext[1]{Corresponding author}

\begin{abstract}	
We introduce \emph{Contractor Filters} (CFs), a set-membership architecture
for nonlinear continuous-time state estimation with bounded uncertainties and
sporadic outliers. CFs represent feasible states symbolically by compositions of
contractors rather than rematerializing a geometric enclosure after each
update. This provides an abstraction layer between the set-theoretic filtering
equations and their numerical realization, so prediction, correction,
ambiguity, and outlier handling can be expressed within one algebra without
prescribing a state geometry. We derive an exact nominal recursion and a
guaranteed robust enclosure, with validated flow inclusions preserving
soundness. Optional sub-paving compression then provides finite outer
representations.

CFs are evaluated against a robust particle filter on range-only and
terrain-based localization under outliers, probabilistic model mismatch,
association ambiguity, and multimodality. CFs retain the simulated truth in all
tested configurations, while particle-filter reliability degrades under
persistent bias and spatial aliasing. Contractor strength and compression budget control the trade-off between paving
tightness and computational cost, while preserving the enclosure guarantee.
\end{abstract}


\maketitle

\section{Introduction}
\label{sec:intro}

Estimating a robot state becomes difficult when nonlinear continuous-time dynamics, uncertain inputs, discrete and possibly corrupted observations, weak initialization, and ambiguous measurements occur simultaneously. Range-only sensing, indistinguishable landmarks, repeated environmental
features, and map-based measurements may all produce large, non-convex, or
disconnected sets of feasible states. An estimator should therefore preserve all states that remain feasible given the
available information until later data resolve the ambiguity.

We consider nonlinear continuous-time systems
\begin{equation}
	\label{eq:system}
	\left\{
\begin{array}{lr}
		\dot{\mathbf{x}}(t)=\mathbf{f}\big(\mathbf{x}(t),\mathbf{u}(t)\big)
		& \textrm{(evolution)}\\
		\mathbf{g}_k\big(\mathbf{x}(t_k)\big)\in\mathbb{Y}_k,\quad k\in\mathbb{N}
		& \textrm{(measurements)},
\end{array}
	\right.
\end{equation}
where $\mathbf{f}:\mathbb{R}^n\times\mathbb{R}^{n_u}\rightarrow\mathbb{R}^{n}$ is the evolution function, $\mathbf{u}(\cdot)$ is an uncertain input trajectory, $\mathbf{g}_k:\mathbb{R}^{n}\rightarrow\mathbb{R}^{p}$ is the $k$-th observation function, and $\mathbb{Y}_k$ is the bounded set associated with the $k$-th observation.

Bayesian filters, such as Kalman-type filters and Particle Filters (PFs), provide
probabilistic state estimates. PFs can represent nonlinear
and multimodal distributions using weighted samples \cite{Arulampalam2002PF}. A finite particle population, however, approximates a probability distribution rather than exhaustively covering all states compatible with prescribed error bounds, and a weakly sampled mode may disappear during resampling. Box-particle variants improve robustness and sampling efficiency with interval-supported particles \cite{AbdallahGningBonnifait2008BPF,Merlinge2019BRPF}, but remain probabilistic approximations.

Set-membership estimation instead assumes unknown-but-bounded uncertainties and seeks the full feasible state set. Recursive set-membership formulations have long been studied \cite{Bertsekas1971}. Numerical representations include interval boxes and sub-pavings \cite{Moore79,Jaulin2001AppliedIntervalAnalysis,Kieffer2001}, ellipsoids and parallelotopes \cite{Chisci1996Parallelotopes}, zonotopes and constrained zonotopes \cite{Alamo2005Zonotopes,Scott2016ConstrainedZonotopes,Rego2020ConstrainedZonotopes}, and polytopic or hybrid-zonotope descriptions for non-convex sets \cite{WanSharmaSutton2018Polytopic,Siefert2023HybridZonotopes}; see also the recent overview in \cite{RauhEtAl2025_LMCS}.

These methods can provide rigorous outer bounds when all set computations are inclusion-preserving, but the chosen state representation strongly affects both tightness and computational complexity. Simple convex representations are compact but may merge separated feasible components or become pessimistic under nonlinear propagation, whereas unions, partitions, and hybrid representations better preserve non-convexity at the price of a larger explicit description. More generally, recursive set propagation repeatedly replaces exact images or intersections by representable outer sets, producing a wrapping effect that may accumulate over time. Partitioning can reduce this effect \cite{WanJaulin2024Wrapping}.

Interval analysis also addresses nonlinear state estimation through constraint propagation rather than only through geometric set propagation. Inclusion functions and contractors can eliminate inconsistent portions of a domain without excluding feasible values \cite{Moore79,Jaulin2001AppliedIntervalAnalysis}, leading to guaranteed recursive state bounding \cite{Kieffer2001} and robust bounded-error estimators able to tolerate outliers \cite{Jaulin09_Automatica_Robust,Carbonnel2014,MeslemHably2020Outliers,Yang2023OutlierZonotope}. However, in many recursive implementations, the resulting state is still materialized after prediction or correction as a box, a paving, a zonotope, or another explicit representation. Thus, even when the underlying computations are rigorous, repeated materialization of the feasible set can itself introduce conservatism, pessimism.

Continuous-time dynamics introduce a second source of approximation: the prediction step must reliably enclose the flow of the dynamical system. Even if the set-theoretic recursion preserves all feasible states, this alone does not guarantee that every continuous-time trajectory is retained when the
flow map is obtained from a numerical approximation of the ODE.

Validated integration addresses this issue by computing outer enclosures of the flow while accounting for truncation and numerical errors \cite{Nedialkov:99,Revol05,dit2016validated,wilczak2012cr}. Such techniques have already been used for guaranteed continuous-time state estimation with interval, Taylor-model, and zonotopic representations \cite{JaulinVoltera02,RaissiRamdaniCandau2004_Automatica, Combastel2005ZonotopeCT,KiefferWalter2006GuaranteedCT, LinStadtherr2007GuaranteedCT,RamdaniNedialkov2011_NAHS}, including data-association problems~\cite{Brendel2026UncertaintyPropagation}. Thus, guaranteed enclosure of the continuous-time prediction is an established ingredient on which the present construction can rely.

Within the set-membership framework, Contractor Programming provides an algebra for composing nonlinear constraints through box contractors \cite{ChabertJaulin09}; contractors and separators are already used in localization and state-estimation problems \cite{DesrochersPolar2016,jaulin_inner_state2016}, with implementations available in open-source libraries such as Codac \cite{codac_lib}. 

Here, we assemble these ingredients into a causal prediction-update filter in
which contractors are used not only as local constraint operators, but as the
recursively stored representation of the feasible states. This introduces an
abstraction layer between the filtering equations and the numerical procedures
used to evaluate them. A main benefit of this contractor-based abstraction is
that the filtering recursion can remain close to its set-theoretic formulation,
while numerical choices affecting contraction strength and computational cost
remain encapsulated within the contractors. The problem can therefore be
formulated in terms of its constraints and set operations without committing
the filter architecture to a particular numerical solution method. This follows
the constraint-programming philosophy that inspires this work, in which problem
modeling is separated from the algorithms used to enforce the constraints.

Nonlinear inverse images and logical combinations can therefore remain symbolic,
so that the filtering recursion is built as a causal chain of contractor
operations rather than as a sequence of explicitly materialized state sets. In
particular, no prescribed geometric wrapper is imposed after each prediction or
correction step: an explicit outer representation is constructed only when
required for visualization, query, or compression. This contractor-based
abstraction is a central idea of the proposed architecture. With sound
contractors and validated flow inclusions, no state compatible with the
bounded-error model is discarded, although infeasible states may remain.

The contribution has three levels:
\begin{enumerate}
	\item \textbf{Set-theoretic formulation.}
	We derive an exact continuous-time recursion and a guaranteed robust enclosure
	for bounded inputs and sporadic outliers. We also identify the main sources of conservatism caused by uncertain inputs,
	including forgetting which admissible input led to each feasible state and the
	factorization of propagated constraints.

	\item \textbf{Contractor-valued filtering architecture.}
	We define the Contractor Filter as a causal estimator whose state is maintained
	as a symbolic composition of contractors rather than being explicitly
	materialized after each filtering step. Observation update and continuous-time
	prediction use the same inverse-contractor primitive, while intersection, union,
	and relaxed intersection provide the algebra required to combine constraints,
	ambiguities, and outliers. Combined with a validated enclosure and inversion of
	the ODE flow, this yields an end-to-end guaranteed filter, including the
	continuous-time prediction step.

	\item \textbf{Online implementation.}
	We introduce a conservative sub-paving compression mechanism that periodically
	replaces the growing symbolic contractor composition by a reusable finite outer
	representation. This limits the evaluation cost of the accumulated graph while
	preserving the guarantee that no feasible state is discarded.

\end{enumerate}

The architecture is evaluated on controlled range-only and terrain-based localization problems combining outliers, within-bound bias, association ambiguity, and repeated map motifs, and is compared with a robust PF under the same simulated inputs and observations. The paper is organized as follows. Section~\ref{sec:recursive:formulation} derives the continuous-time set-membership recursion for bounded inputs and outliers. Section~\ref{sec:contractors} introduces the contractor concepts and algebra used in the construction. Section~\ref{sec:filter} defines Contractor Filters, their guaranteed semantics, and the recursive sub-paving compression mechanism. Section~\ref{sec:applications} presents the range-only and terrain-based evaluation and the comparison with particle filtering. Finally, Section~\ref{sec:conclusion} summarizes the main results and discusses future extensions.

\section{Recursive formulation of the feasible states}
\label{sec:recursive:formulation}

We first define the continuous-time feasible states set independently of any interval approximation, then derive the nominal and outlier-robust recursions.
With bounded inputs, exact feasibility depends on the complete state and input trajectories. Projecting this information onto a current state set loses the association between a state and the input history that made it feasible. This distinction matters in the robust case.

\subsection{Tubes}
\label{sec:tubes}

Denote by $\mathcal{F}(\mathbb{R}\rightarrow\mathbb{R}^{n})$ the set of $\mathbb{R}^{n}$-valued trajectories. A tube~\cite{kurzhanski_theory_1993} is a set-valued map
\[
\mathbb{X}(\cdot):\mathbb{R}\rightarrow\mathcal{P}(\mathbb{R}^{n})
\]
that represents the set of trajectories satisfying
\begin{equation}
	\label{eq:tube:membership}
	\mathbf{x}(\cdot)\in\mathbb{X}(\cdot)
	\quad\Longleftrightarrow\quad
	\forall t,\;
	\mathbf{x}(t)\in\mathbb{X}(t).
\end{equation}
Each slice $\mathbb{X}(t)$ encloses the admissible values at time $t$; the tube is pointwise and does not by itself encode temporal relations between slices. For example, the constant tube $\mathbb{A}(t)=[-1,1]$ contains every trajectory remaining in $[-1,1]$, but does not distinguish the subset of nondecreasing trajectories, since $\dot a(t)\geqslant0$ couples values at different times. Such temporal relations are imposed separately by the dynamics or by tube contractors~\cite{RohouTubeval}.

\subsection{Exact sequence without considering outliers}
\label{sec:exact:nooutliers}

At each measurement time $t_k$, we denote by $\mathbb{X}^{-}_{t_k}$ the feasible prior immediately before the update, and by $\mathbb{X}^{+}_{t_k}$ the feasible set after the update. In the nominal case,
\begin{equation}
	\label{eq:nominal:posterior:set}
	\mathbb{X}^{+}_{t_k}
	=
	\mathbb{X}^{-}_{t_k}
	\cap
	\mathbf{g}_{k}^{-1}(\mathbb{Y}_{k}).
\end{equation}
If no initialization information is available, the prior is simply $\mathbb{X}^{-}_{t_0}=\mathbb{R}^{n}$; no artificial finite initial box is required. Subsequent constraints may naturally reduce this set whenever the available information is informative.

We first consider the nominal case in which every measurement is assumed consistent. For a discrete-time system, recursive set-membership estimation can be written as~\cite{Bertsekas1971}
\begin{equation}
	\label{eq:exact:sequence}
	\mathbb{X}_{k+1}
	=\mathbf{f}_{k}(\mathbb{X}_{k})
	\cap(\mathbf{f}_{k}\circ\mathbf{g}_{k}^{-1})(\mathbb{Y}_{k}).
\end{equation}

For the continuous-time system~\eqref{eq:system}, consider an initial state $\mathbf{x}_{1}$ at time $t_{1}$. For a fixed input trajectory $\mathbf{u}(\cdot)$, the state at $t_{2}$ is
\begin{equation}
	\label{eq:flow}
	\mathbf{x}_{2}
	=
	\bm{\Phi}_{t_{1},t_{2}}^{\mathbf{u}(\cdot)}
	(\mathbf{x}_{1}),
\end{equation}
where $\bm{\Phi}_{t_{1},t_{2}}^{\mathbf{u}(\cdot)}$ denotes the flow of the ODE over $[t_{1},t_{2}]$. Its set extension, or reachable-set map, is
\begin{equation}
	\label{eq:reachable:set}
	\bm{\Phi}_{t_1,t_2}^{\mathbf{u}(\cdot)}(\mathbb{X}_1)
	=
	\left\{
	\bm{\Phi}_{t_1,t_2}^{\mathbf{u}(\cdot)}(\mathbf{x}_1)
	\mid \mathbf{x}_1\in\mathbb{X}_1
	\right\}.
\end{equation}

For a fixed input trajectory $\mathbf{u}(\cdot)$, assume that solutions are
unique on $[t_1,t_2]$. The corresponding flow
$\bm{\Phi}_{t_1,t_2}^{\mathbf{u}(\cdot)}$ is then injective: two distinct
initial states cannot reach the same state at time $t_2$ under the same input. An injective map preserves intersections; hence, for any $\mathbb{X}_1^a,\mathbb{X}_1^b\subseteq\mathbb{R}^n$,
\begin{equation}
	\label{eq:automorphism}
	\bm{\Phi}_{t_1,t_2}^{\mathbf{u}(\cdot)}
	(\mathbb{X}_1^a\cap\mathbb{X}_1^b)
	=
	\bm{\Phi}_{t_1,t_2}^{\mathbf{u}(\cdot)}(\mathbb{X}_1^a)
	\cap
	\bm{\Phi}_{t_1,t_2}^{\mathbf{u}(\cdot)}(\mathbb{X}_1^b).
\end{equation}

Using~\eqref{eq:automorphism}, the known-input recursion reduces to the propagation of the corrected set $\mathbb{X}^{+}_{t_k}$ defined in~\eqref{eq:nominal:posterior:set}:
\begin{equation}
	\label{eq:known:input:update:first}
	\mathbb{X}(t)=
	\bm{\Phi}_{t_k,t}^{\mathbf{u}(\cdot)}
	\big(\mathbb{X}^{+}_{t_k}\big).
\end{equation}

We now let the input be uncertain, but bounded by a tube $\mathbb{U}(\cdot)$. The corresponding set-valued flow is defined by
\begin{equation}
	\label{eq:reachable:set:U}
	\bm{\Phi}_{t_1,t_2}^{\mathbb{U}(\cdot)}(\mathbb{X}_1)
	=
	\bigcup_{\mathbf{u}(\cdot)\in\mathbb{U}(\cdot)}
	\bm{\Phi}_{t_1,t_2}^{\mathbf{u}(\cdot)}(\mathbb{X}_1).
\end{equation}

\begin{proposition}[Exact recursion]
\label{th:exact:sequence:nooutliers}

Consider system~\eqref{eq:system} with bounded inputs. Let $\mathbb{U}(\cdot)$ be the input tube and $\mathbb{T}=\{t_0,\ldots,t_{\bar k}\}$ the measurement times.
For $t\in[t_0,t_{\bar k}]$, let $t_k$ be the last measurement time not exceeding $t$. Then the exact feasible state set satisfies
\begin{align}
		\mathbb{X}(t)
		&=
		\bm{\Phi}_{t_k,t}^{\mathbb{U}(\cdot)}
		\big(\mathbb{X}^{+}_{t_k}\big)
		\label{th:exact:sequence:nooutliers:eq}
		\\
		&=
		\bigcup_{\mathbf{u}(\cdot)\in\mathbb{U}(\cdot)}
		\bm{\Phi}_{t_k,t}^{\mathbf{u}(\cdot)}
		\left(
		\mathbb{X}^{-}_{t_k}
		\cap
		\mathbf{g}_k^{-1}(\mathbb{Y}_k)
		\right).
		\label{th:exact:sequence:nooutliers:expanded}
\end{align}
\end{proposition}

\begin{proof}
Let $\mathbf{x}\in\mathbb{R}^n$. By definition of feasibility, $\mathbf{x}\in\mathbb{X}(t)$ if and only if there exist $\mathbf{u}(\cdot)\in\mathbb{U}(\cdot)$ and $\mathbf{x}_k\in\mathbb{X}^{+}_{t_k}$ such that $\mathbf{x}= \bm{\Phi}_{t_k,t}^{\mathbf{u}(\cdot)}(\mathbf{x}_k)$. Therefore,
\[
\mathbf{x}\in\mathbb{X}(t) \quad\Longleftrightarrow\quad \mathbf{x}\in \bigcup_{\mathbf{u}(\cdot)\in\mathbb{U}(\cdot)} \bm{\Phi}_{t_k,t}^{\mathbf{u}(\cdot)} \big(\mathbb{X}^{+}_{t_k}\big).
\]
Using~\eqref{eq:reachable:set:U} and $\mathbb{X}^{+}_{t_k} = \mathbb{X}^{-}_{t_k}\cap\mathbf{g}_k^{-1}(\mathbb{Y}_k)$ from \eqref{eq:nominal:posterior:set} gives \eqref{th:exact:sequence:nooutliers:eq}, \eqref{th:exact:sequence:nooutliers:expanded}.
\end{proof}

Proposition~\ref{th:exact:sequence:nooutliers} is exact in the set-theoretic sense: it characterizes precisely the states reachable from the corrected set under at least one admissible input trajectory. Factorizing the corrected set into independently propagated terms under uncertain inputs loses this common-input dependency and yields only an outer enclosure.

\begin{proposition}[Factorized enclosure]
\label{prop:factorized:bounded:input}

For a non-degenerate input tube $\mathbb{U}(\cdot)$, the exact recursion admits the outer enclosure
\begin{equation}
		\label{eq:filter:nooutliers:multi:u}
		\mathbb{X}(t)
		\subseteq
		\bm{\Phi}_{t_k,t}^{\mathbb{U}(\cdot)}(\mathbb{X}^{-}_{t_k})
		\cap
		(\bm{\Phi}_{t_k,t}^{\mathbb{U}(\cdot)}
		\circ\mathbf{g}_{k}^{-1})(\mathbb{Y}_{k}).
\end{equation}
The inclusion may be strict even when every fixed-input flow is injective.
\end{proposition}

\begin{proof}
Set $\mathbb{A}=\mathbb{X}^{-}_{t_k}$ and $\mathbb{B}=\mathbf{g}_k^{-1}(\mathbb{Y}_k)$. Under injective fixed-input flows, Eq.~\eqref{eq:automorphism} yields
\begin{equation}
\begin{aligned}
			&\bigcup_{\mathbf{u}(\cdot)\in\mathbb{U}(\cdot)}
			\bm{\Phi}^{\mathbf{u}(\cdot)}_{t_k,t}
			(\mathbb{A}\cap\mathbb{B})
			\\
			&=
			\bigcup_{\mathbf{u}(\cdot)\in\mathbb{U}(\cdot)}
			\left[
			\bm{\Phi}^{\mathbf{u}(\cdot)}_{t_k,t}(\mathbb{A})
			\cap
			\bm{\Phi}^{\mathbf{u}(\cdot)}_{t_k,t}(\mathbb{B})
			\right]
			\\
			&\subseteq
			\left[
			\bigcup_{\mathbf{u}(\cdot)\in\mathbb{U}(\cdot)}
			\bm{\Phi}^{\mathbf{u}(\cdot)}_{t_k,t}(\mathbb{A})
			\right]\cap
			\left[
			\bigcup_{\mathbf{v}(\cdot)\in\mathbb{U}(\cdot)}
			\bm{\Phi}^{\mathbf{v}(\cdot)}_{t_k,t}(\mathbb{B})
			\right].
\end{aligned}
\end{equation}

The last intersection may use different admissible input trajectories in its two factors, whereas feasibility requires one common trajectory; this is exactly the relaxation in~\eqref{eq:filter:nooutliers:multi:u}.
\end{proof}

Hence~\eqref{th:exact:sequence:nooutliers:eq} remains the exact reference, whereas~\eqref{eq:filter:nooutliers:multi:u} is a convenient outer enclosure.

\subsection{Recursive set-membership estimation with outliers}
\label{sec:exact:outliers}

A measurement set $\mathbb{Y}_k$ is an outlier when it does not contain the actual output, namely when $\mathbf{g}_k(\mathbf{x}(t_k))\notin\mathbb{Y}_k$. Outlier rejection has been considered in static bounded-error estimation problems~\cite{Sliwka09,Carbonnel2014} and in recursive state estimation settings~\cite{Jaulin09_Automatica_Robust,MeslemHably2020Outliers}.

We retain the historical MNO terminology of~\cite{Jaulin09_Automatica_Robust}, which in the present context denotes the \emph{Maximum Number of Outliers}. Operationally, $q$ is the maximum number of violated measurement constraints allowed in any window of $\ell$ measurements, with $0\leqslant q<\ell$; equivalently, at least $\ell-q$ constraints in each window must be satisfied.

For the window ending at index $k$, this is written as
\begin{equation}
	\label{eq:mno:conjunction}
	\mathop{\bigwedge\nolimits^{\{q\}}}_{i\in\{0,\ldots,\ell-1\}}
	\bigg(
	\mathbf{g}_{k-i}
	\big(\mathbf{x}(t_{k-i})\big)
	\in
	\mathbb{Y}_{k-i}
	\bigg),
\end{equation}
where $\bigwedge^{\{q\}}$ denotes a $q$-relaxed conjunction. For example, allowing one outlier among three measurements gives $ \bigwedge\nolimits^{\{1\}}(P_1,P_2,P_3) \quad\Longleftrightarrow\quad (P_1\wedge P_2) \vee (P_1\wedge P_3) \vee (P_2\wedge P_3), $ where \(P_i\) denotes satisfaction of the \(i\)-th measurement constraint.

At the set level, the corresponding operator is the $q$-relaxed intersection.

\begin{dfn}
Consider $\ell$ sets $\mathbb{X}_1,\ldots,\mathbb{X}_\ell\subset\mathbb{R}^n$. Their $q$-relaxed intersection contains the vectors belonging to at least $\ell-q$ of these sets:
\begin{equation}
		\label{eq:qintersection:def}
		\mathbf{x}\in
		\mathop{\bigcap\nolimits^{\{q\}}}_{i\in\{1,\ldots,\ell\}}\mathbb{X}_i
		\quad\Longleftrightarrow\quad
		\#\{i:\mathbf{x}\in\mathbb{X}_i\}
		\geqslant\ell-q.
\end{equation}
In particular, $\mathop{\bigcap\nolimits^{\{0\}}}_{i\in\{1,\ldots,\ell\}}\mathbb{X}_i =\bigcap_{i=1}^{\ell}\mathbb{X}_i$.
\end{dfn}

Using this notation, the discrete-time robust recursion of \cite[Theorem~2.1]{Jaulin09_Automatica_Robust} is
\begin{equation}
	\label{eq:exact:sequence:outliers}
\begin{aligned}
		\mathbb{X}_{k+1}
		={}\mathbf{f}_k(\mathbb{X}_k)
		\cap
		\mathop{\bigcap\nolimits^{\{q\}}}_{i\in\{0,\ldots,\ell-1\}}
		\left(
		\mathbf{f}_k^i
		\circ
		\mathbf{g}_{k-i}^{-1}
		\right)
		(\mathbb{Y}_{k-i}),
\end{aligned}
\end{equation}
where $\mathbf{f}_k^i=\mathbf{f}_k\circ\mathbf{f}_{k-1}\circ\cdots \circ\mathbf{f}_{k-i}$.

With bounded inputs, the exact robust feasible set is defined by the full state and input trajectories. Since a recursion storing only $\mathbb{X}^{-}_{t_k}$ and $\mathbb{U}(\cdot)$ does not retain the full dependency between the current state and its past input trajectory, we derive guaranteed state-space enclosures rather than an exact projected recursion.

\begin{theorem}[Robust enclosure]
\label{th:exact:sequence:outliers}

Let $t_k$ be the last measurement time not exceeding $t$, with $k\geqslant\ell-1$ so that a complete window of $\ell$ measurements is available. Assume the MNO condition with at most $q$ outliers in this window. Then the exact feasible state set satisfies
\begin{equation}
		\label{eq:robust:common:u}
\begin{aligned}
			\mathbb{X}(t)
			\subseteq{}&
			\bigcup_{\mathbf{u}(\cdot)\in\mathbb{U}(\cdot)}
			\Bigg[
			\bm{\Phi}_{t_k,t}^{\mathbf{u}(\cdot)}
			(\mathbb{X}^{-}_{t_k})
			\\
			&\qquad\cap
			\mathop{\bigcap\nolimits^{\{q\}}}_{i\in\{0,\ldots,\ell-1\}}
			\left(
			\bm{\Phi}_{t_{k-i},t}^{\mathbf{u}(\cdot)}
			\circ\mathbf{g}_{k-i}^{-1}
			\right)(\mathbb{Y}_{k-i})
			\Bigg].
\end{aligned}
\end{equation}

The same admissible input trajectory $\mathbf{u}(\cdot)$ is used in every flow
term of~\eqref{eq:robust:common:u}, restricted to the corresponding time
interval. Replacing these common-input flows by independently propagated
tube-level flows gives the directly usable factorized enclosure
\begin{equation}
		\label{th:exact:sequence:outliers:eq}
\begin{aligned}
			\mathbb{X}(t)
			\subseteq{}&
			\bm{\Phi}_{t_k,t}^{\mathbb{U}(\cdot)}
			(\mathbb{X}^{-}_{t_k})
			\\
			&\cap
			\mathop{\bigcap\nolimits^{\{q\}}}_{i\in\{0,\ldots,\ell-1\}}
			\left(
			\bm{\Phi}_{t_{k-i},t}^{\mathbb{U}(\cdot)}
			\circ
			\mathbf{g}_{k-i}^{-1}
			\right)
			(\mathbb{Y}_{k-i}).
\end{aligned}
\end{equation}
\end{theorem}

\begin{proof}
Take any $\mathbf{x}\in\mathbb{X}(t)$. By feasibility, there exists one admissible input trajectory $\mathbf{u}(\cdot)\in\mathbb{U}(\cdot)$ and one associated state trajectory such that $\mathbf{x}(t_k)\in\mathbb{X}^{-}_{t_k}$ and at least $\ell-q$ measurement constraints in the current window are satisfied. For each satisfied constraint at time $t_{k-i}$, propagation with this same input trajectory places $\mathbf{x}$ in $(\bm{\Phi}_{t_{k-i},t}^{\mathbf{u}(\cdot)} \circ\mathbf{g}_{k-i}^{-1})(\mathbb{Y}_{k-i})$. Hence $\mathbf{x}$ belongs to the right-hand side of \eqref{eq:robust:common:u}.

For every set $\mathbb A$ and every admissible $\mathbf u(\cdot)\in\mathbb U(\cdot)$, $\bm\Phi^{\mathbf u(\cdot)}(\mathbb A)\subseteq\bm\Phi^{\mathbb U(\cdot)}(\mathbb A)$. Monotonicity of ordinary and $q$-relaxed intersections therefore gives the factorized enclosure~\eqref{th:exact:sequence:outliers:eq}.
\end{proof}

Equation~\eqref{eq:robust:common:u} already loses the association between the projected prior state and the past input history that made it feasible. The factorized form~\eqref{th:exact:sequence:outliers:eq} adds a second relaxation because its propagated factors may use different admissible trajectories from the same input tube. This factorized enclosure is the one implemented by the CF.

Fig.~\ref{fig:exact:sequence:outliers} illustrates the robust enclosure.
\begin{figure}[h]
	\centering
	\includegraphics[width=\linewidth]{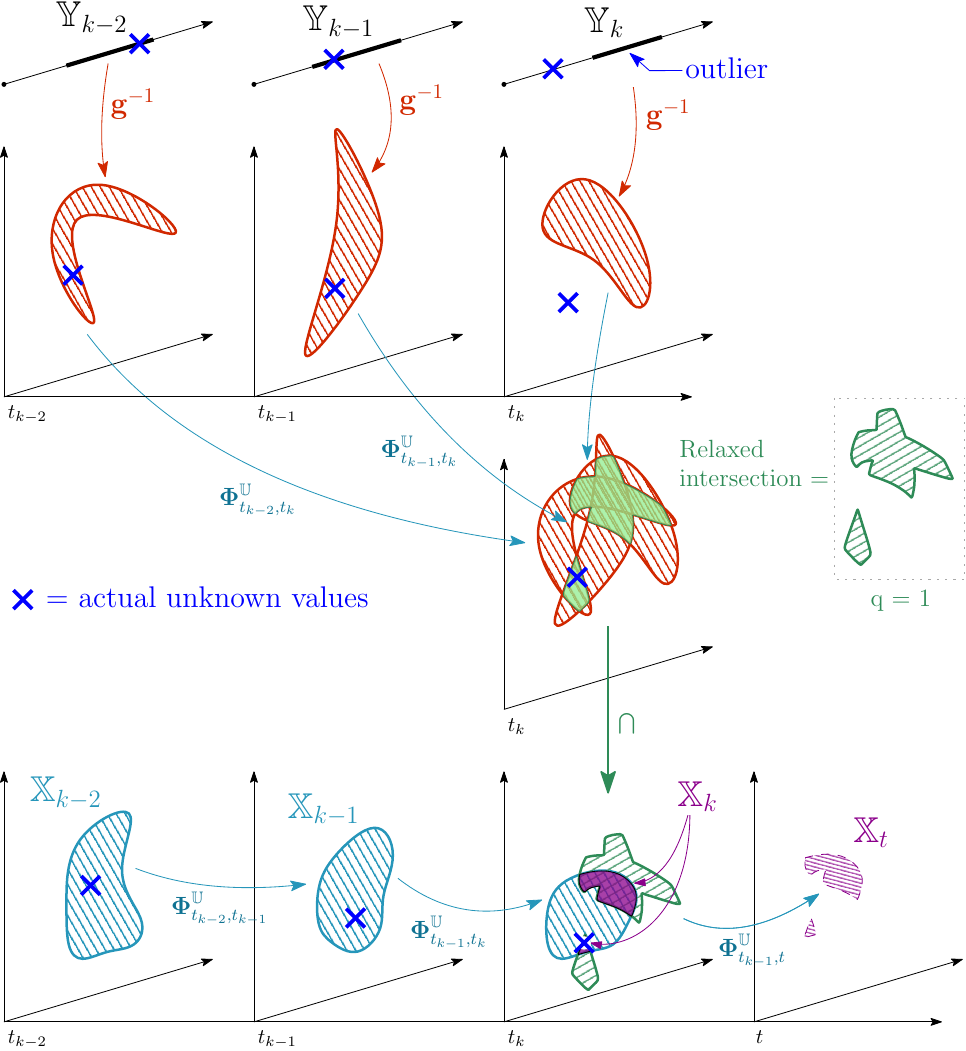}
	\caption{Illustration of
		Theorem~\ref{th:exact:sequence:outliers}: robust state-space enclosure in
		the presence of an outlier. Here, $\ell=3$ and $q=1$.}
	\label{fig:exact:sequence:outliers}
\end{figure}

\subsection{Flow of nonlinear systems}
\label{sec:flow}

The previous results are written in terms of exact flow maps. A numerical CF requires a computable flow inclusion containing every trajectory generated by the initial set and input tube. This requirement is essential for preserving the set-membership guarantee through the prediction step. We propose to exploit the cascade structure available in many robotics models. General validated ODE integration remains available when such a structure does not exist.

\subsubsection{Strict triangular systems}

Recall the nonlinear time-invariant system $\dot{\mathbf{x}}=\mathbf{f}(\mathbf{x},\mathbf{u})$. An analytical flow is generally unavailable. A useful special case is a strict triangular system, for which the state variables can be ordered as
\begin{equation}
	\label{eq:triangular}
\begin{aligned}
		\dot{x}_1 &= f_1(\mathbf{u}),\\
		\dot{x}_2 &= f_2(x_1,\mathbf{u}),\\
		&\ \vdots\\
		\dot{x}_n &= f_n(x_1,\ldots,x_{n-1},\mathbf{u}).
\end{aligned}
\end{equation}

In this causal structure,  each equation depends only on the input and on state components that can already be propagated. This cascade permits sequential upstream-to-downstream propagation without solving a mutually coupled fixed-point problem over all state components~\cite{RohouTubint}.

Consider, for example, a wheeled robot with position $(x_1,x_2)^\intercal$, heading $x_3$, and speed $x_4$:
\begin{equation}
	\label{eq:triangular:wheeled}
\begin{aligned}
		\dot{x}_3 &= u_1,
		&\qquad
		\dot{x}_4 &= u_2,\\
		\dot{x}_1 &= x_4\cos(x_3),
		&
		\dot{x}_2 &= x_4\sin(x_3).
\end{aligned}
\end{equation}
With ordering $(x_3,x_4,x_1,x_2)$, heading and speed are propagated before position, yielding the cascade in Fig.~\ref{fig:causal:kinematic:chain}.

\begin{figure}[h]
	\centering
	\includegraphics[width=0.8\linewidth]{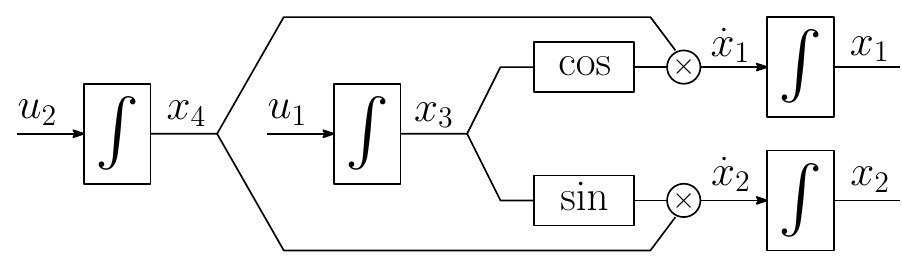}
	\caption{Cascade representation of the wheeled-robot dynamics
		in~\eqref{eq:triangular:wheeled}.}
	\label{fig:causal:kinematic:chain}
\end{figure}

For a strict triangular system, interval propagation follows the same upstream-to-downstream order. When the associated constraint graph is a tree and the elementary propagation operators are exact, constraint propagation can recover the tightest interval enclosures compatible with the constraints \cite{Montanari91}. In general, interval dependency may still introduce over-approximation, but no fixed-point iteration is needed solely to resolve algebraic coupling between state equations.

When the dynamics cannot be put in strict triangular form, general validated ODE solvers can instead be used, based for example on interval Runge-Kutta schemes, Taylor methods, or Lohner-type techniques \cite{Nedialkov:99,dit2016validated,Revol05,wilczak2012cr}. Their use in bounded-error state estimation has been studied notably in~\cite{JaulinVoltera02}. These methods preserve the required inclusion guarantee, although tightness and computational cost may deteriorate for wide initial sets or uncertain inputs. The triangular implementation is therefore not a requirement of the CF architecture, but an efficient validated-flow implementation when the model permits it.

\subsubsection{The inputs are enclosed in a tube}
\label{sec:flow:input:tubes}

Let $\mathbb{IR}^n$ denote the set of interval boxes in $\mathbb R^n$.
On $T_k=[t_k,t_{k+1}]$, let $[\mathbf{u}](\cdot):T_k\rightarrow\mathbb{IR}^{n_u}$ be a tube of boxes satisfying $\mathbb{U}(t)\subseteq[\mathbf{u}](t)$. Tube arithmetic and validated tube integration can then be used to construct reliable flow inclusions~\cite{RohouTubint}. In particular, tube integration returns an interval vector enclosing the integral of every trajectory contained in the tube~\cite{aubry:autom:13}. For the wheeled robot of~\eqref{eq:triangular:wheeled}, define $\mathbf{p}=(x_1,x_2)^\intercal$, $\mathbf{q}=(x_3,x_4)^\intercal$, so that $\mathbf{x}=(\mathbf{p},\mathbf{q})$ and $\mathbf{u}=(u_1,u_2)^\intercal$. The strict triangular structure first integrates $\mathbf{u}$ to obtain the heading/speed increment, and then uses this result to integrate the planar velocity. 

\begin{algorithm}[t]
	\caption{Flow inclusion for the wheeled-robot model}
	\label{alg:wheeled:flow}
	\small
\begin{algorithmic}[1]

		\State \textbf{Input:}
		$T_k=[t_k,t_{k+1}]$,
		$[\mathbf{u}](\cdot):T_k\rightarrow\mathbb{IR}^{2}$

		\State \textbf{Output:}
		$[\bm{\Phi}]^{[\mathbf{u}](\cdot)}_{t_k,t}:\mathbb{IR}^4\to\mathbb{IR}^4$,
		$\quad t\in T_k$

		\State
		$[\Delta\mathbf{q}](\tau)
		\gets
		\displaystyle\int_{t_k}^{\tau}[\mathbf{u}](s)\,ds$,
		$\quad \tau\in T_k$

		\State
		$\displaystyle
\begin{aligned}
			{}[\Delta\mathbf{p}]([\mathbf{q}],t)
			\gets {}&
			\int_{t_k}^{t}
			\bigl([q_2]+[\Delta q_2](\tau)\bigr)\\
			&\qquad\cdot
\begin{pmatrix}
				\cos\!\bigl([q_1]+[\Delta q_1](\tau)\bigr)\\
				\sin\!\bigl([q_1]+[\Delta q_1](\tau)\bigr)
\end{pmatrix}
			\,d\tau
\end{aligned}$

		\State \Return
		$\displaystyle
\begin{aligned}
			{}[\bm{\Phi}]^{[\mathbf{u}](\cdot)}_{t_k,t}
			:\;&[\mathbf{p}]\times[\mathbf{q}]\mapsto\\
			&
			\bigl([\mathbf{p}]+[\Delta\mathbf{p}]([\mathbf{q}],t)\bigr)
			\times
			\bigl([\mathbf{q}]+[\Delta\mathbf{q}](t)\bigr)
\end{aligned}$

\end{algorithmic}
\end{algorithm}

Algorithm~\ref{alg:wheeled:flow} constructs the corresponding inclusion function for every $t\in T_k$.
The returned object is not a single propagated box, but an inclusion function parameterized by the input tube $[\mathbf{u}](\cdot)$, which maps each provided state box $[\mathbf{x}]$ to an enclosure of its reachable states. Therefore, for every $t\in T_k$, every admissible input trajectory $\mathbf{u}(\cdot)$ belonging to $\mathbb{U}(\cdot)$, and every $\mathbf{x}(t_k)\in[\mathbf{x}]$,
\begin{equation}
	\bm{\Phi}_{t_k,t}^{\mathbf{u}(\cdot)}
	\bigl(\mathbf{x}(t_k)\bigr)
	\in
	[\bm{\Phi}]^{[\mathbf{u}](\cdot)}_{t_k,t}
	([\mathbf{x}]).
	\label{eq:flow:inclusion}
\end{equation}
Since $[\Delta\mathbf{q}](\cdot)$ is independent of the source box, it can be
precomputed once over $T_k$ and reused in all flow evaluations. In the general
case, the planar increment $[\Delta\mathbf{p}]([\mathbf{q}],t)$ must still be
evaluated for each source box because it depends on the initial heading and
speed. This triangular structure nevertheless makes the computation efficient.

The same construction can exploit a guaranteed tube on intermediate components,
for instance from bounded continuous measurements of heading or speed. Applying
the differential contractor $\mathcal{C}_{d/dt}$ of
~\cite{RohouTubint,linobs} to $\dot{\mathbf q}=\mathbf u$ contracts
$[\mathbf q](\cdot)$ using the bounded input tube $[\mathbf u](\cdot)$.
Since $\mathbf q$ is upstream of the planar dynamics, this contraction directly
tightens the displacement integral and the resulting flow enclosure. It also
removes the remaining source-box dependence from the displacement computation:
the contracted tube and corresponding displacement integrals can be precomputed
once over $T_k$. Subsequent flow evaluations then require only interval-vector
intersections and an addition. This conditioned refinement is used in the
experiments below.

The filtering construction of Section~\ref{sec:filter} also uses validated inclusions of backward flows. On one measurement interval $T_k$, for $t\in T_k$, we denote by $[\bm{\Phi}]^{[\mathbf{u}](\cdot)}_{t,t_k}$ any backward-flow inclusion satisfying
\begin{equation}
	\label{eq:backward:flow:inclusion}
	\bigcup_{\mathbf{u}(\cdot)\in\mathbb{U}(\cdot)}
	\bm{\Phi}^{\mathbf{u}(\cdot)}_{t,t_k}
	([\mathbf{x}_{t}])
	\subseteq
	[\bm{\Phi}]^{[\mathbf{u}](\cdot)}_{t,t_k}
	([\mathbf{x}_{t}]).
\end{equation}
For triangular systems this enclosure follows from the same relations in reverse time; otherwise a validated ODE solver is applied backward.

We now introduce the contractor representation used to encode and combine the feasible-state constraints in the filter.

\section{Catalog of contractors}
\label{sec:contractors}

Contractors are standard objects of interval constraint programming \cite{Jaulin2001AppliedIntervalAnalysis,ChabertJaulin09} and provide the computational layer used in the next section to translate the feasible-set recursions into operator compositions. We first recall their box-contraction semantics and consistency with the represented set, then introduce the contractor algebra and inverse-contractor constructions used for measurement updates, prediction, ambiguity, and outlier handling.

\subsection{Definition}

\begin{dfn}
	A contractor associated with a set $\mathbb X\subseteq\mathbb R^n$ is a
	box-to-box operator $\mathcal C:\mathbb{IR}^n\to\mathbb{IR}^n$ such that, for
	every box $[\mathbf x]\in\mathbb{IR}^n$,
	\begin{enumerate}
		\item[(i)] contraction:\quad
		$\mathcal C([\mathbf x])\subseteq[\mathbf x]$;
		\item[(ii)] consistency:\quad
		$\mathbb X\cap[\mathbf x]\subseteq\mathcal C([\mathbf x])$.
	\end{enumerate}
	We then say that $\mathcal C$ is a sound contractor representation of the set $\mathbb X$.
\end{dfn}

We use the notation
\begin{equation}
	\label{eq:set:represented:by:ctc}
	\mathbb X\trianglelefteq\mathcal C
	\quad\Longleftrightarrow\quad
	\forall[\mathbf x]\in\mathbb{IR}^n,\;
	\mathbb X\cap[\mathbf x]\subseteq\mathcal C([\mathbf x]).
\end{equation}
This relation expresses reliability, not contraction strength: different
contractors may be consistent with the same set while producing different
contracted boxes.

\begin{example}
Consider
\[
\mathbb X = \left\{ \mathbf x\in\mathbb R^2 \mid x_1-\exp(x_2)=0 \right\}.
\]
A contractor associated with $\mathbb X$ is
\begin{equation}
	\mathcal C([\mathbf x])
	=
	\left(
	[x_1]\cap\exp([x_2]),
	\;
	[x_2]\cap\ln([x_1])
	\right).
\end{equation}
\end{example}

Property~(i) ensures that contraction can only discard parts of the input box.
Property~(ii) is the guarantee property: every feasible vector contained in the
input box remains in the contracted box. A contractor may therefore be weak (it
may fail to remove some infeasible vectors) without losing reliability. More
precisely, the minimal contractor associated with $\mathbb X$ maps a query box
$[\mathbf x]$ to $\square(\mathbb X\cap[\mathbf x])$, the smallest box
enclosing all feasible points of the query. A non-minimal contractor may retain
additional infeasible regions: this affects contraction strength, but not
soundness. This distinction will be important for the TBN experiments of
Section~\ref{sec:applications:tbn}.

Contractors can be embedded in branch-and-contract algorithms, commonly called
pavers. Starting from an initial box $[\mathbf{x}_{0}]$, the paver contracts
candidate boxes and bisects those that cannot yet be classified at the requested
resolution. If $\mathbb{X}\trianglelefteq\mathcal{C}$, every region removed by
contraction is certified not to contain a point of $\mathbb{X}$. Consequently,
the union $\mathbb{S}_{\mathcal{P}}$ of the terminal boxes of a sound
branch-and-contract procedure satisfies
\begin{equation}
	\mathbb{X}\cap[\mathbf{x}_{0}]
	\subseteq
	\mathbb{S}_{\mathcal{P}}.
\end{equation}
If the process stops below a prescribed width $\varepsilon$, this yields a
finite-resolution outer approximation of the represented set within the search
domain.

\begin{example}
	\label{ex:ctc}
	Figure~\ref{fig:example:ctc} shows static range-only localization from two bounded distances $[y_1]$ and $[y_2]$, i.e., the set inversion $\mathbf g(\mathbf x)\in[\mathbf y]$. A reliable outer approximation of the solution set $\mathbb X$ is obtained by paving the search domain with a sound contractor $\mathcal C_{\mathbf g}$.
\end{example}

\begin{figure}[h]
	\centering
	\includegraphics[width=\linewidth]{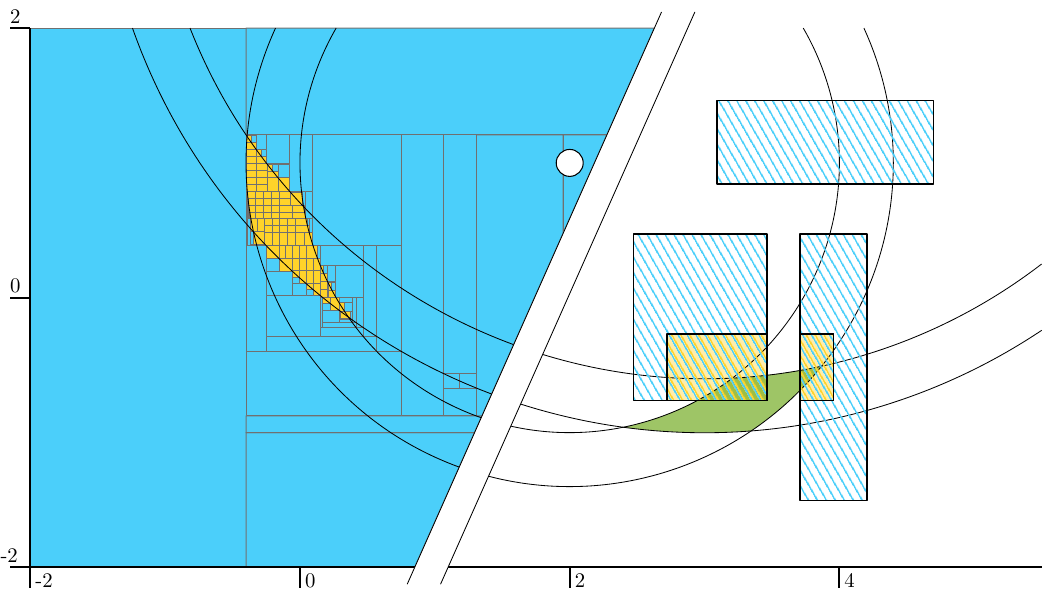}
	\caption{
		Two-distance localization (Example~\ref{ex:ctc}). Right: action of the
		contractor $\mathcal{C}_{\mathbf{g}}$ on single boxes. Left: outer sub-paving approximating
		$\operatorname{set}(\mathcal C_{\mathbf g})$, obtained by repeated contractions and subdivisions.
		Yellow boxes may contain feasible positions, whereas blue regions are proven
		infeasible.
	}
	\label{fig:example:ctc}
\end{figure}

\subsection{Algebra of contractors}

A contractor also induces an implicit set through the points that it cannot reject,
\begin{equation}
	\label{eq:set:of:ctc}
	\operatorname{set}(\mathcal C)
	=
	\left\{
	\mathbf x\in\mathbb R^n
	\mid
	\mathcal C(\{\mathbf x\})=\{\mathbf x\}
	\right\}.
\end{equation}
Here a point $\mathbf{x}$ is identified with its degenerate box $\{\mathbf x\}$. If
$\mathbb X\trianglelefteq\mathcal C$, then
$\mathbb X\subseteq\operatorname{set}(\mathcal C)$. Thus a contractor has a
dual interpretation: operationally, it is a box-to-box contraction operator;
semantically, it represents the implicit set $\operatorname{set}(\mathcal C)$. This dual view is central to Contractor Programming~\cite{ChabertJaulin09}.
It provides an abstraction layer in which sets can be manipulated through
a common contractor interface, while the numerical method used to perform each
contraction remains encapsulated.
Figure~\ref{fig:example:ctc} illustrates both interpretations: the right panel
shows $\mathcal C_{\mathbf g}$ acting as a box-to-box operator, while the left panel visualizes an outer approximation of
$\operatorname{set}(\mathcal C_{\mathbf g})$ obtained by paving.

Accordingly, set operations can be lifted directly to contractors while retaining a box-to-box implementation.
Denoting by $\mathfrak{C}_n$ the class of contractors on $\mathbb{IR}^n$, an $m$-ary contractor operation maps $m$ contractors to another contractor: $\mathfrak{O}: \mathfrak{C}_n\times\cdots\times\mathfrak{C}_n \rightarrow \mathfrak{C}_n$. Once its arguments are fixed, $\mathfrak{O}(\mathcal C_1,\ldots,\mathcal C_m)$ is an ordinary box-to-box contractor. For instance, at the box-operator level, the intersection and union of two contractors are defined by
\begin{align}
	(\mathcal C_1\cap\mathcal C_2)([\mathbf x])
	&=\mathcal C_1([\mathbf x])\cap\mathcal C_2([\mathbf x]),
	\label{eq:ctc:inter}\\
	(\mathcal C_1\cup\mathcal C_2)([\mathbf x])
	&=\square\!\left(
	\mathcal C_1([\mathbf x])\cup
	\mathcal C_2([\mathbf x])
	\right),
	\label{eq:ctc:union}
\end{align}
where $\square(\cdot)$ denotes the interval hull. The hull is required because the union of two boxes is not generally a box, whereas a contractor must return an element of $\mathbb{IR}^n$.

At point resolution, these two contractor operations exactly reproduce the corresponding set operations:
\begin{align}
	\operatorname{set}(\mathcal C_1\cap\mathcal C_2)
	&=\operatorname{set}(\mathcal C_1)\cap\operatorname{set}(\mathcal C_2),\\
	\operatorname{set}(\mathcal C_1\cup\mathcal C_2)
	&=\operatorname{set}(\mathcal C_1)\cup\operatorname{set}(\mathcal C_2).
\end{align}
This is the set-theoretic semantics behind the contractor algebra. The resulting objects remain ordinary box-to-box operators.

Therefore, if $\mathbb A\trianglelefteq\mathcal C_1$ and $\mathbb B\trianglelefteq\mathcal C_2$, then
\begin{align}
	\mathbb A\cap\mathbb B
	&\trianglelefteq\mathcal C_1\cap\mathcal C_2,
	\label{eq:ctc:set:inter:sound}\\
	\mathbb A\cup\mathbb B
	&\trianglelefteq\mathcal C_1\cup\mathcal C_2.
	\label{eq:ctc:set:union:sound}
\end{align}
The interval hull of~\eqref{eq:ctc:union} may lose disconnected geometric information within a single
box evaluation, but this does not affect soundness. When the contractor is
embedded in a paver, subdivision can recover this disconnected structure at the paving resolution
through the resulting union of boxes. The box-valued formulation remains
computationally convenient: interval boxes provide a simple lattice structure
for intersection and hull, together with efficient inclusion tests and
bisection.

This viewpoint also makes constraint composition straightforward.
The problem in Example~\ref{ex:ctc} amounts to a conjunction of two scalar
distance constraints, so almost the same contraction is obtained by intersecting
the two corresponding distance contractors.

\subsection{Inverse contractors}
\label{sec:contractors:inverse}

The range-only problem of Example~\ref{ex:ctc} is an inverse problem: the feasible positions are the preimage of the bounded measurement set, $\mathbb X=\mathbf g^{-1}([\mathbf y])$. Equivalently, when the scalar distance constraints are treated separately, each distance contractor contracts a source box with respect to the inverse image of one bounded range interval. The inverse-contractor construction generalizes this principle to arbitrary mappings and to target sets themselves represented by contractors.

Consider a mapping $\mathbf h:\mathbb R^n\rightarrow\mathbb R^p$ and an inclusion function $[\mathbf h]:\mathbb{IR}^n\rightarrow\mathbb{IR}^p$ satisfying $\forall [\mathbf x],~\mathbf{h}([\mathbf{x}])\subseteq[\mathbf{h}]([\mathbf{x}])$. Given a target contractor $\mathcal C_{\mathbb Y}\in\mathfrak C_p$ representing $\mathbb Y\subseteq\mathbb R^p$, we define the inverse-contractor operation
\begin{equation}
	\label{eq:cf:inverse:operation}
	\operatorname{Inv}_{[\mathbf h]}:
	\mathfrak C_p\rightarrow\mathfrak C_n,
\end{equation}
such that
\begin{equation}
	\label{eq:cf:inverse:soundness}
	\mathbb Y\trianglelefteq\mathcal C_{\mathbb Y}
	\quad\Rightarrow\quad
	\mathbf h^{-1}(\mathbb Y)
	\trianglelefteq
	\operatorname{Inv}_{[\mathbf h]}
	(\mathcal C_{\mathbb Y}).
\end{equation}
Once the target contractor is fixed, $\operatorname{Inv}_{[\mathbf h]}(\mathcal C_{\mathbb Y})$ is an ordinary box-to-box contractor on the input space.

An HC4Revise-style forward-backward implementation \cite{BenhamouICLP99} first evaluates $[\mathbf y]=[\mathbf h]([\mathbf x])$, then contracts the root domain according to $[\mathbf y]\leftarrow\mathcal C_{\mathbb Y}([\mathbf y])$, and propagates this restriction backward through the computational graph to contract the input domains $[x_i]$.

In the contractor framework, the important point is that the target $\mathbb Y$ is itself represented by a contractor rather than being restricted to an explicit box or other prescribed geometric representation. The same inversion primitive can therefore be applied to a bounded box, a non-convex set represented by a composition of contractors, or another contractor-valued state estimate. Such representations can therefore be combined and inverted without first materializing intermediate feasible sets as boxes or other prescribed geometric enclosures. This avoids an additional
source of over-approximation while preserving the soundness property
in~\eqref{eq:cf:inverse:soundness}.

\subsection{Contractor for relaxed intersection}

The robust recursion in~\eqref{th:exact:sequence:outliers:eq} additionally requires a contractor counterpart of the $q$-relaxed intersection defined in~\eqref{eq:qintersection:def}. Here, $q$ denotes the maximum number of violated constraints among $\ell$ constraints, so a feasible point must belong to at least $\ell-q$ of the corresponding sets.

For boxes $[\mathbf b_1],\ldots,[\mathbf b_\ell]$, define
\begin{equation}
	\label{eq:qinter:boxes}
	\mathcal Q_q([\mathbf b_1],\ldots,[\mathbf b_\ell])
	=
	\square\!\left\{
	\mathbf x\mid
	\#\{i:\mathbf x\in[\mathbf b_i]\}\geqslant \ell-q
	\right\}.
\end{equation}
Thus, $\mathcal Q_q$ is the smallest box containing all points that belong to at least $\ell-q$ input boxes.

Let $\mathbb X_i\trianglelefteq\mathcal C_i$ for $i=1,\ldots,\ell$. The operation
\begin{equation}
	\label{eq:ctc:qinter:operation}
	\operatorname{QInter}_q:
	\mathfrak C_n\times\cdots\times\mathfrak C_n
	\rightarrow
	\mathfrak C_n
\end{equation}
maps $\ell$ contractors to a new contractor. Its action on a query box is
\begin{equation}
	\label{eq:ctc:qinter}
	\operatorname{QInter}_q
	(\mathcal C_1,\ldots,\mathcal C_\ell)([\mathbf x])
	=
	\mathcal Q_q\!\left(
	\mathcal C_1([\mathbf x]),\ldots,\mathcal C_\ell([\mathbf x])
	\right).
\end{equation}

It is contracting and satisfies
\begin{equation}
	\label{eq:ctc:qinter:set}
	\mathop{\bigcap\nolimits^{\{q\}}}_{i=1}^{\ell}\mathbb X_i
	\trianglelefteq
	\operatorname{QInter}_q(\mathcal C_1,\ldots,\mathcal C_\ell).
\end{equation}
For $q=0$, this is the ordinary contractor intersection.

\begin{example}
	\label{ex:qinter}
	The range-only localization problem of Example~\ref{ex:ctc} is extended to three beacons. Suppose that one of the three bounded range measurements may be an outlier, without knowing which one. With $\ell=3$ and $q=1$, $\operatorname{QInter}_1$ combines the three distance contractors and retains all states consistent with at least two of the three measurements. Figure~\ref{fig:q:inter} illustrates the resulting relaxed contraction.
\end{example}

\begin{figure}
	\centering
	\includegraphics[width=\linewidth]{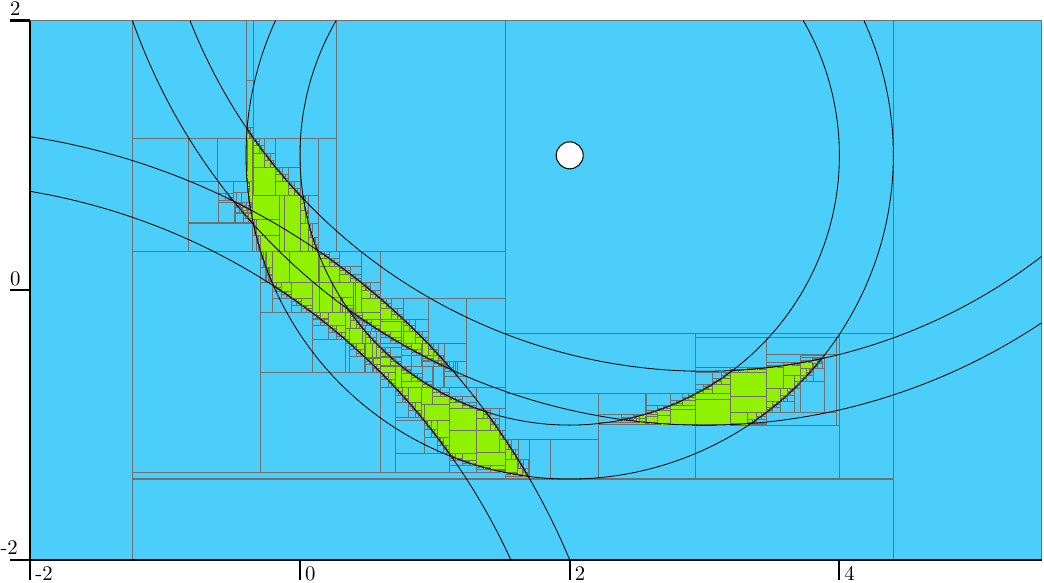}
	\caption{Example~\ref{ex:ctc} with a third beacon and one admissible outlier.
		With $\ell=3$ and $q=1$, $\operatorname{QInter}_1$ retains states consistent
		with at least two range constraints. Inner information is shown only for
		illustration and is not used by the CF.}
	\label{fig:q:inter}
\end{figure}

Computing the box operator~\eqref{eq:qinter:boxes} is a combinatorial problem. Exact and heuristic algorithms are given in~\cite{Carbonnel2014}.

\medskip

These operations provide the algebraic counterparts of the set operations in Section~\ref{sec:recursive:formulation}. In the next section, the feasible-state information is represented by a contractor $\mathcal{C}_{\mathbb{X}(t)}$ satisfying $\mathbb{X}(t)\trianglelefteq\mathcal{C}_{\mathbb{X}(t)}$. This contractor itself becomes the recursive filter state, rather than materializing a new explicit set after every update.


\section{Contractor Filters for Nonlinear Continuous-Time Systems}
\label{sec:filter}

We now translate the feasible-set recursions into a causal contractor-valued estimator.
The contractor abstraction allows this translation to closely follow the
set-theoretic equations, without committing the filter recursion to a particular
geometric state representation or numerical inversion method.
The inverse-contractor operation introduced in Section~\ref{sec:contractors:inverse} provides a common primitive for both measurement update and prediction.

\subsection{Contractor-valued recursion}

\paragraph{Observation inversion.} Let $\mathcal C_{\mathbb Y_k}\in\mathfrak C_p$ represent the measurement set $\mathbb Y_k$. The observation contractor is
\begin{equation}
	\mathcal C_k^{\mathrm{obs}}
	\triangleq
	\operatorname{Inv}_{[\mathbf g_k]}
	(\mathcal C_{\mathbb Y_k}).
	\label{eq:ctc:obs}
\end{equation}
In the common case where the measurement set is a box $\mathbb Y_k=[\mathbf y_k]$, its natural contractor is simply
\[
\mathcal C_{\mathbb Y_k}([\mathbf y]) = [\mathbf y]\cap[\mathbf y_k].
\]
The observation contractor then directly contracts the state with respect to the bounded measurement constraint $\mathbf g_k(\mathbf x)\in[\mathbf y_k]$.

\paragraph{Evolution as set inversion.} Prediction uses the same construction. For each admissible input trajectory, the ODE flow is the mapping $\bm{\Phi}^{\mathbf{u}(\cdot)}_{t_k,t}: \mathbb{R}^{n}\rightarrow\mathbb{R}^{n}$. Under existence and uniqueness, its inverse is the backward flow, hence
\begin{equation}
	\label{eq:cf:flow:image:as:inverse}
	\bm{\Phi}^{\mathbb{U}(\cdot)}_{t_k,t}(\mathbb{A})
	=
	\bigcup_{\mathbf{u}(\cdot)\in\mathbb{U}(\cdot)}
	\left(
	\bm{\Phi}^{\mathbf{u}(\cdot)}_{t,t_k}
	\right)^{-1}(\mathbb{A}).
\end{equation}
Prediction is therefore set inversion of the backward flow. On each measurement interval, let $[\bm{\Phi}]^{[\mathbf{u}](\cdot)}_{t,t_k}$ be the validated backward-flow inclusion of~\eqref{eq:backward:flow:inclusion}; local inclusions are composed across several intervals. We require
\begin{equation}
	\label{eq:cf:evolution:soundness}
	\mathbb{A}\trianglelefteq\mathcal{C}_{\mathbb{A}}
	\quad\Rightarrow\quad
	\bm{\Phi}^{\mathbb{U}(\cdot)}_{t_k,t}(\mathbb{A})
	\trianglelefteq
	\operatorname{Inv}_{
		[\bm{\Phi}]^{[\mathbf{u}](\cdot)}_{t,t_k}}
	\left(
	\mathcal{C}_{\mathbb{A}}
	\right).
\end{equation}
Thus, in the Contractor Programming framework, observation update and prediction are two instances of the same inverse-contractor algorithm: only the inclusion function and the target contractor change, while the contractor-valued representation is preserved throughout the composition.

\begin{dfn}[Contractor Filter]
\label{def:contractor:filter}
A Contractor Filter (CF) for system~\eqref{eq:system} is a causal family $\{\mathcal{C}_{\mathbb{X}(t)}\}_{t\geqslant t_0}$ satisfying
\begin{equation}
		\label{eq:cf:guarantee:def}
		\mathbb{X}(t)
		\trianglelefteq
		\mathcal{C}_{\mathbb{X}(t)}
		\qquad\forall t\geqslant t_0.
\end{equation}
\end{dfn}

We construct such a filter recursively using inverse contractors for observation and evolution together with the contractor algebra of Section~\ref{sec:contractors}.

\paragraph{Remark (A tube of contractors).}
A Contractor Filter may be viewed as a contractor-valued generalization of a
set-valued state tube,
\[
t\mapsto\mathcal C_{\mathbb X(t)}.
\]
Since each contractor $\mathcal C_{\mathbb X(t)}$ induces the set
$\operatorname{set}(\mathcal C_{\mathbb X(t)})\subseteq\mathbb R^n$, the
CF also induces the set-valued tube
\[
t\mapsto\operatorname{set}(\mathcal C_{\mathbb X(t)}),
\]
which encloses the feasible state set $\mathbb X(t)$ at each time. In particular, a CF is not merely a sequence of contractors defined at measurement times: between two successive updates, the validated flow construction defines $\mathcal C_{\mathbb X(t)}$ for any queried time $t\in[t_k,t_{k+1}]$. We emphasize that this contractor-valued tube should not be confused with a contractor acting on tube-valued variables~\cite{RohouTubint}.

Let $\mathcal C^-_{t_k}$ and $\mathcal C^+_{t_k}$ represent the prior and corrected sets. The update is
\begin{equation}
	\label{eq:cf:updated:prior}
	\mathcal{C}^{+}_{t_k}
	=
	\mathcal{C}^{-}_{t_k}
	\cap
	\mathcal{C}^{\mathrm{obs}}_{k}.
\end{equation}
With no initialization information, $\mathbb X^-_{t_0}=\mathbb R^n$ is represented by the identity contractor $\mathcal I$ (i.e., $\mathcal I([\mathbf x])=[\mathbf x]$); no artificial finite prior is required. For $t\in[t_k,t_{k+1})$,
\begin{equation}
	\label{eq:ctcfilter}
	\mathcal{C}_{\mathbb{X}(t)}
	=
	\operatorname{Inv}_{
		[\bm{\Phi}]^{[\mathbf{u}](\cdot)}_{t,t_k}}
	\left(
	\mathcal{C}^{+}_{t_k}
	\right).
\end{equation}
Algorithm~\ref{alg:cf:nominal} summarizes the nominal recursion and optional compression.

\begin{proposition}[Nominal CF soundness]
	\label{prop:cf:soundness}
	
	Assume a sound prior
	$\mathbb{X}^{-}_{t_k}\trianglelefteq\mathcal{C}^{-}_{t_k}$,
	and a sound measurement representation
	$\mathbb{Y}_{k}\trianglelefteq\mathcal{C}_{\mathbb{Y}_{k}}$.
	Assume also that the observation inverse contractor satisfies the soundness
	property~\eqref{eq:cf:inverse:soundness} and that the validated backward-flow
	inversion satisfies~\eqref{eq:cf:evolution:soundness}. Then, for every
	$t\in[t_k,t_{k+1})$,
	\begin{equation}
	\mathbb X(t)
	\trianglelefteq
	\operatorname{Inv}_{
		[\bm{\Phi}]^{[\mathbf u](\cdot)}_{t,t_k}}
	\left(
	\mathcal C^-_{t_k}
	\cap
	\operatorname{Inv}_{[\mathbf g_k]}
	(\mathcal C_{\mathbb Y_k})
	\right)
	=
	\mathcal C_{\mathbb X(t)}.
	\end{equation}
	Hence the recursion defines a Contractor Filter in the sense of
	Definition~\ref{def:contractor:filter}.
\end{proposition}

\begin{proof}
From $\mathbb{Y}_{k}\trianglelefteq\mathcal{C}_{\mathbb{Y}_{k}}$ and \eqref{eq:ctc:obs},
$\mathbf{g}_{k}^{-1}(\mathbb{Y}_{k}) \trianglelefteq \mathcal{C}^{\mathrm{obs}}_{k}$.
Together with $\mathbb{X}^{-}_{t_k}\trianglelefteq\mathcal{C}^{-}_{t_k}$, the soundness of contractor intersection gives
\[
\mathbb{X}^{+}_{t_k} = \mathbb{X}^{-}_{t_k} \cap \mathbf{g}_{k}^{-1}(\mathbb{Y}_{k}) \trianglelefteq \mathcal{C}^{-}_{t_k} \cap \mathcal{C}^{\mathrm{obs}}_{k} = \mathcal{C}^{+}_{t_k}.
\]
Applying~\eqref{eq:cf:evolution:soundness} and the exact bounded-input recursion of Proposition~\ref{th:exact:sequence:nooutliers} yields
\[
\mathbb{X}(t) \trianglelefteq \operatorname{Inv}_{ [\bm{\Phi}]^{[\mathbf{u}](\cdot)}_{t,t_k}} \left(\mathcal{C}^{+}_{t_k}\right) = \mathcal{C}_{\mathbb{X}(t)}.
\]
Hence the recursion satisfies~\eqref{eq:cf:guarantee:def}.
\end{proof}

\paragraph{Operational interpretation.} At any time $t$, $\mathcal C_{\mathbb X(t)}$ is the top-level operator of a contractor-composition graph encoding the filtering recursion up to that time. Querying it with a box recursively activates the underlying contractors; the box may be rejected directly, or further refined by a paver through subdivision and repeated contractions. Hence the CF preserves the symbolic constraint structure instead of materializing a geometric enclosure after each update. With sufficiently tight contractors, paving can therefore approximate
the exact nominal feasible set in~\eqref{th:exact:sequence:nooutliers:eq}
up to the paving resolution and the over-approximation introduced by the
contractors and the validated flow inclusion. This avoids an additional
geometric materialization step after every filtering update.

\subsection{{Recursive implementation with sub-paving compression}}
\label{sec:recursive:implementation}
\label{sec:subpaving:compression}

The symbolic contractor graph grows with each update, so that retaining the complete history increases the number of operations traversed by each query to the current top-level contractor. We therefore optionally compress this history by replacing the subgraph encoding earlier filtering steps with a contractor backed by a finite outer sub-paving. Subsequent queries use this paving instead of traversing the discarded contractor composition. This is primarily a change of representation, trading graph traversal for paving storage and querying, and does not necessarily reduce memory usage.

Let $\mathbb{D}_k$ be a bounded working domain satisfying $\mathbb{X}^{+}_{t_k}\subseteq\mathbb{D}_k$. For a sound state contractor $\mathbb{X}^{+}_{t_k}\trianglelefteq\mathcal{C}^{+}_{t_k}$, we denote by
\begin{equation}
	\label{eq:compression:paving}
	\mathcal{P}_k
	=
	\operatorname{Pave}
	\left(
	\mathbb{D}_k,
	\mathcal{C}^{+}_{t_k},
	\varepsilon_k
	\right)
	=
	\{[\mathbf{b}^{(j)}_k]\}_{j=1}^{N_k}
\end{equation}
a finite outer sub-paving produced by a sound branch-and-contract procedure at resolution $\varepsilon_k$, such that
\begin{equation}
	\label{eq:compression:outer}
	\mathbb{X}^{+}_{t_k}\cap\mathbb{D}_k
	\subseteq
	\mathbb{S}_{\mathcal{P}_k},
	\qquad
	\mathbb{S}_{\mathcal{P}_k}
	\triangleq
	\bigcup_{j=1}^{N_k}[\mathbf{b}^{(j)}_k].
\end{equation}
Since $\mathbb{X}^{+}_{t_k}\subseteq\mathbb{D}_k$, this reduces to $\mathbb{X}^{+}_{t_k}\subseteq\mathbb{S}_{\mathcal{P}_k}$.

The corresponding sub-paving contractor is
\begin{equation}
    \label{eq:paving:contractor}
    \mathcal{C}_{\mathcal{P}_k}([\mathbf{x}])
    =
    \square\!\left(
        \mathbb{S}_{\mathcal{P}_k}\cap[\mathbf{x}]
    \right)
    =
    \square\!\left(
        \bigcup_{j=1}^{N_k}
        \big([\mathbf{x}]\cap[\mathbf{b}^{(j)}_k]\big)
    \right),
\end{equation}
so
\begin{equation}
    \label{eq:paving:contractor:sound}
    \mathbb{S}_{\mathcal{P}_k}
    \trianglelefteq
    \mathcal{C}_{\mathcal{P}_k}.
\end{equation}
The stored representation remains the finite sub-paving $\mathcal P_k$. The interval hull is computed only when evaluating a query box. The paving boxes may be indexed by an appropriate spatial data structure for efficient queries.

Although~\eqref{eq:compression:paving} is written with a target resolution $\varepsilon_k$, the same soundness argument applies to a time-budgeted paver provided that unresolved boxes are retained in the returned outer paving. This conservative time-budgeted variant is used in Section~\ref{sec:applications}.

\begin{proposition}[Compression soundness]
\label{prop:cf:compression:soundness}
Assume $\mathbb{X}^{+}_{t_k}\trianglelefteq\mathcal{C}^{+}_{t_k}$ and $\mathbb{X}^{+}_{t_k}\subseteq\mathbb{D}_k$, and let $\mathcal{P}_k$ be returned by the sound paving operation in~\eqref{eq:compression:paving}. Then
\begin{equation}
        \label{eq:compression:soundness}
        \mathbb{X}^{+}_{t_k}
        \trianglelefteq
        \mathcal{C}_{\mathcal{P}_k}.
\end{equation}
Consequently, replacing $\mathcal{C}^{+}_{t_k}$ by $\mathcal{C}_{\mathcal{P}_k}$ preserves the CF enclosure guarantee.
\end{proposition}

\begin{proof}
By~\eqref{eq:compression:outer}, $\mathbb X^+_{t_k}\subseteq\mathbb S_{\mathcal P_k}$. Hence, for every query box $[\mathbf x]$,
\[
\mathbb X^+_{t_k}\cap[\mathbf x] \subseteq \mathbb S_{\mathcal P_k}\cap[\mathbf x] \subseteq \mathcal C_{\mathcal P_k}([\mathbf x]),
\]
by the definition of the paving contractor. Therefore, $\mathbb X^+_{t_k}\trianglelefteq\mathcal C_{\mathcal P_k}$.
\end{proof}

Algorithm~\ref{alg:cf:nominal} makes explicit the compression step, in which the accumulated symbolic contractor composition is replaced by the sub-paving contractor $\mathcal C_{\mathcal P_k}$.
Figure~\ref{fig:cf:graph:nominal} shows the nominal CF recursion and its optional compression as a contractor graph.

\begin{algorithm}[t]
\caption{Nominal Contractor Filter}
\label{alg:cf:nominal}
\begin{algorithmic}[1]
\State \textbf{Input:} $\mathcal{C}^{+}_{t_{k-1}}$, $\mathcal{C}_{\mathbb{Y}_k}$,
$[\mathbf{g}_k]$, $[\bm{\Phi}]^{[\mathbf{u}](\cdot)}_{t_k,t_{k-1}}$,
$\mathbb{D}_k$, $\varepsilon_k$
\State \textbf{Output:} $\mathcal{C}^{+}_{t_k}$
\State $\mathcal{C}^{-}_{t_k}
\gets
\operatorname{Inv}_{
[\bm{\Phi}]^{[\mathbf{u}](\cdot)}_{t_k,t_{k-1}}}
\left(
\mathcal{C}^{+}_{t_{k-1}}
\right)$
\State $\mathcal{C}^{+}_{t_k}
\gets
\mathcal{C}^{-}_{t_k}
\cap
\operatorname{Inv}_{[\mathbf{g}_k]}
\left(
\mathcal{C}_{\mathbb{Y}_k}
\right)$
\If{compression criterion is met}
    \State $\mathcal{P}_k
    \gets
    \operatorname{Pave}
    \left(
    \mathbb{D}_k,
    \mathcal{C}^{+}_{t_k},
    \varepsilon_k
    \right)$
    \State $\mathcal{C}^{+}_{t_k}
    \gets
    \mathcal{C}_{\mathcal{P}_k}$
\EndIf
\end{algorithmic}
\end{algorithm}

\begin{figure*}[t]
    \centering
    {\includegraphics[width=0.8\textwidth]{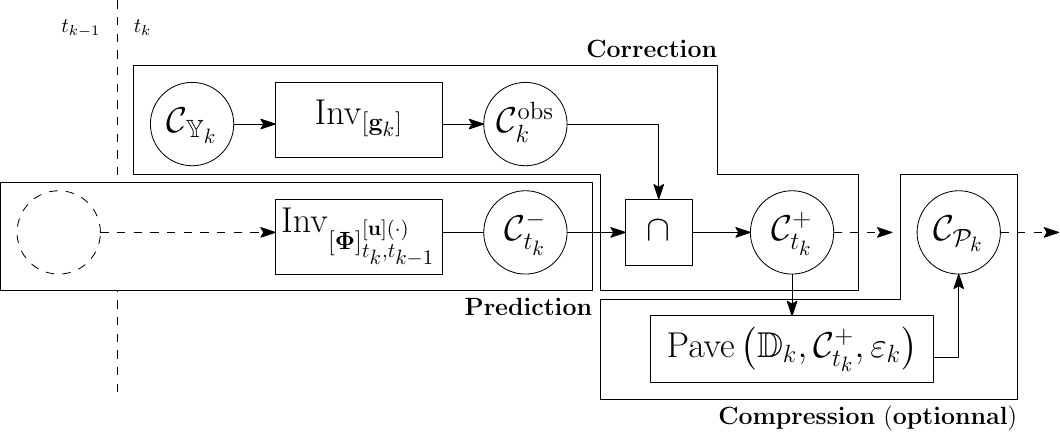}}
    \caption{
    Contractor graph of one nominal CF recursion, showing the prediction,
    correction, and optional sub-paving compression stages. The dashed
    continuation indicates the recursive link to the previous and next filtering steps.
    When compression is activated, the materialized contractor
    $\mathcal{C}_{\mathcal{P}_k}$ replaces the accumulated symbolic graph.}
    \label{fig:cf:graph:nominal}
\end{figure*}

Compression may be triggered periodically or according to graph depth, evaluation time, or memory usage. After compression, query cost depends on the paving size and resolution rather than on the depth of the discarded graph.

\subsection{Robust Contractor Filters}
\label{sec:robust:ctc:filter}

The robust CF implements the factorized enclosure \eqref{th:exact:sequence:outliers:eq}. For $k\geqslant\ell-1$, define the MNO window
\[
\mathbb W_k=\{k-\ell+1,\ldots,k\}.
\]
For each $i\in\mathbb W_k$, let $\mathcal C_{i\rightarrow k}$ denote the contractor associated with observation $i$, expressed in the state space at the current time $t_k$. For a past observation $i<k$, this contractor is obtained by transporting the observation constraint from $t_i$ to $t_k$ through the validated flow inclusion:
\begin{equation}
	\label{eq:cf:robust:propagated:observation}
	\mathcal{C}_{i\rightarrow k}
	\triangleq
	\operatorname{Inv}_{
		[\bm{\Phi}]^{[\mathbf{u}](\cdot)}_{t_k,t_i}}
	\left(
	\mathcal{C}^{\mathrm{obs}}_i
	\right),
	\qquad
	i\in\mathbb{W}_k\setminus\{k\}.
\end{equation}
For $i=k$, no transport is required, since the observation already constrains the current state directly: $\mathcal C_{k\rightarrow k}\triangleq\mathcal C_k^{\mathrm{obs}}.$

The factorized robust update is then
\begin{equation}
    \label{eq:ctcfilter:robust}
    \mathcal{C}^{+}_{t_k}
    =
    \mathcal{C}^{-}_{t_k}
    \cap
    \operatorname{QInter}_{q}
    \left(
        (\mathcal{C}_{i\rightarrow k})_{i\in\mathbb{W}_k}
    \right).
\end{equation}
Between measurements, the nominal prediction~\eqref{eq:ctcfilter} applies. Algorithm~\ref{alg:cf:robust} summarizes the robust recursion.

\begin{algorithm}[t]
\caption{MNO-Robust Contractor Filter}
\label{alg:cf:robust}
\begin{algorithmic}[1]
\State \textbf{Input:} $\mathcal{C}^{+}_{t_{k-1}}$,
$\{\mathcal{C}_{\mathbb{Y}_i},[\mathbf{g}_i]\}_{i\in\mathbb{W}_k}$,
$[\bm{\Phi}]^{[\mathbf{u}](\cdot)}$, $\ell$, $q$,
$\mathbb{D}_k$, $\varepsilon_k$
\State \textbf{Output:} $\mathcal{C}^{+}_{t_k}$
\State $\mathcal{C}^{-}_{t_k}
\gets
\operatorname{Inv}_{
[\bm{\Phi}]^{[\mathbf{u}](\cdot)}_{t_k,t_{k-1}}}
\left(
\mathcal{C}^{+}_{t_{k-1}}
\right)$
\State $\mathcal{C}_{k\rightarrow k}
\gets
\operatorname{Inv}_{[\mathbf{g}_k]}
\left(
\mathcal{C}_{\mathbb{Y}_k}
\right)$
\For{$i\in\mathbb{W}_k\setminus\{k\}$}
    \State $\mathcal{C}_{i\rightarrow k}
    \gets
    \operatorname{Inv}_{
    [\bm{\Phi}]^{[\mathbf{u}](\cdot)}_{t_k,t_i}}
    \left(
    \operatorname{Inv}_{[\mathbf{g}_i]}
    \left(
    \mathcal{C}_{\mathbb{Y}_i}
    \right)
    \right)$
\EndFor
\State $\mathcal{C}^{+}_{t_k}
\gets
\mathcal{C}^{-}_{t_k}
\cap
\operatorname{QInter}_{q}
\left(
(\mathcal{C}_{i\rightarrow k})_{i\in\mathbb{W}_k}
\right)$
\If{compression criterion is met}
    \State $\mathcal{P}_k
    \gets
    \operatorname{Pave}
    \left(
    \mathbb{D}_k,
    \mathcal{C}^{+}_{t_k},
    \varepsilon_k
    \right)$
    \State $\mathcal{C}^{+}_{t_k}
    \gets
    \mathcal{C}_{\mathcal{P}_k}$
\EndIf
\end{algorithmic}
\end{algorithm}

Figure~\ref{fig:cf:graph:robust} shows the robust update for $\ell=3$. The robust formulation introduces the two parameters $(\ell,q)$: the MNO-window length and the maximum number of violated observation constraints within that window. As in the nominal CF, the recursive state contractor may be compressed by sub-paving. This compression replaces only the recursive state graph. Recent observation contractors and local validated flow inclusions are retained as long as they are needed to construct subsequent MNO windows.

\begin{figure*}[t]
    \centering
    {\includegraphics[width=0.8\textwidth]{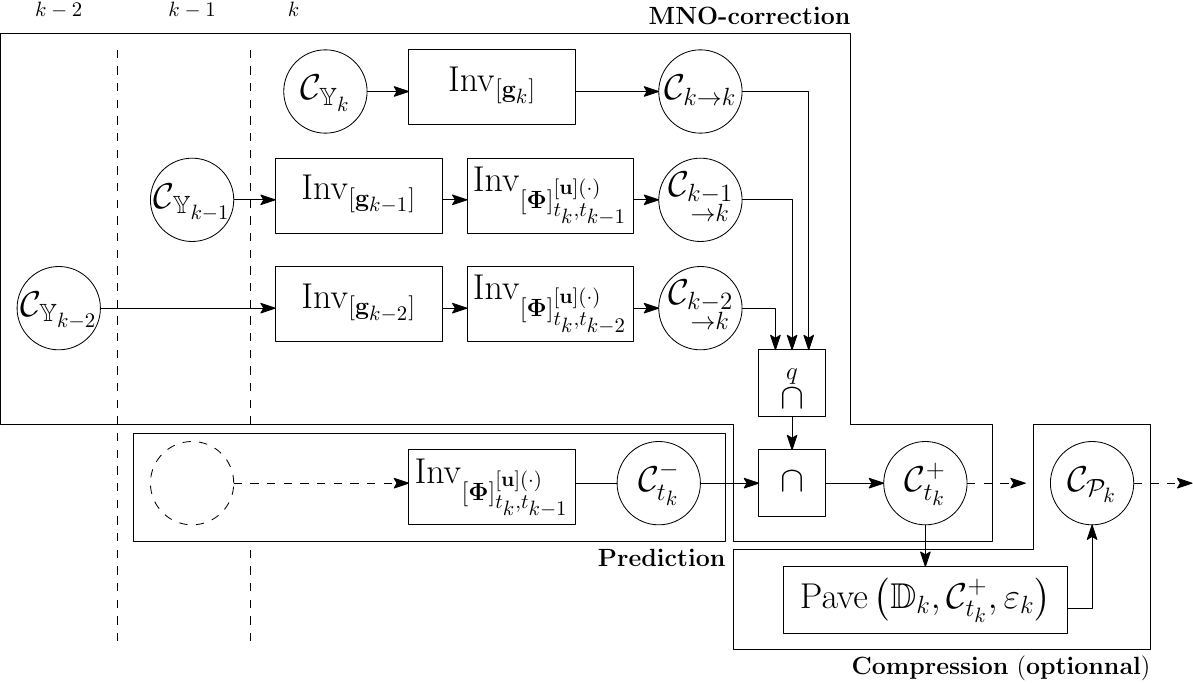}}
    \caption{
    Contractor graph of an MNO-robust CF update, illustrated for
    $\ell=3$, together with optional sub-paving compression. The propagated
    observation contractors are first combined by the $q$-relaxed
    intersection. It is then intersected with the prediction
    contractor $\mathcal{C}^{-}_{t_k}$ to form
    $\mathcal{C}^{+}_{t_k}$.}
    \label{fig:cf:graph:robust}
\end{figure*}

\begin{proposition}[Robust CF soundness]
\label{prop:cf:robust:soundness}
Assume the MNO condition and a sound prior $\mathbb{X}^{-}_{t_k}\trianglelefteq\mathcal{C}^{-}_{t_k}$. Assume also sound observation and validated flow inversions, together with the relaxed-intersection property~\eqref{eq:ctc:qinter:set}.
Then the robust update satisfies
\begin{equation}
		\label{eq:cf:robust:posterior:soundness}
		\mathbb{X}(t_k)
		\trianglelefteq
		\mathcal{C}^{+}_{t_k}.
\end{equation}
Subsequent prediction by~\eqref{eq:ctcfilter} preserves the CF guarantee for $t\in[t_k,t_{k+1})$.
\end{proposition}

\begin{proof}
For each $i\in\mathbb W_k$, observation inversion and validated transport give a sound contractor $\mathcal C_{i\rightarrow k}$ for the corresponding observation constraint expressed at $t_k$. Hence, by~\eqref{eq:ctc:qinter:set}, their $q$-relaxed intersection soundly encloses the relaxed observation constraint of the factorized enclosure~\eqref{th:exact:sequence:outliers:eq}. Together with $\mathbb{X}^{-}_{t_k}\trianglelefteq\mathcal{C}^{-}_{t_k}$ and the soundness of contractor intersection, this yields
\[
\mathbb X(t_k) \trianglelefteq \mathcal C^-_{t_k} \cap \operatorname{QInter}_q \left( (\mathcal C_{i\rightarrow k})_{i\in\mathbb W_k} \right) = \mathcal C^+_{t_k}.
\]
Applying the nominal prediction~\eqref{eq:ctcfilter} and the evolution soundness property~\eqref{eq:cf:evolution:soundness} then gives $\mathbb X(t)\trianglelefteq\mathcal C_{\mathbb X(t)}$ for all $t\in[t_k,t_{k+1})$.
\end{proof}

Thus, the robust CF provides a contractor realization of the factorized outer
enclosure of Theorem~\ref{th:exact:sequence:outliers}. If sub-paving compression
is applied, Proposition~\ref{prop:cf:compression:soundness} preserves the same
enclosure guarantee.

\paragraph{Separator extension.} Separators may provide inner information for static inverse-image problems \cite{DesrochersEAAI2014,DesrochersPolar2016,jaulin_inner_state2016}. Extending this property to the present dynamic setting is less direct, since prediction relies on tube and validated-flow representations that provide outer enclosures but do not, in general, propagate certified inner information. A Separator Filter would therefore require a dedicated inner-reachability formalism and is left for future work. The separator information shown in Fig.~\ref{fig:q:inter} is used only for illustration on a static example, and is not part of the CF recursion.

\FloatBarrier
\section{Applications and evaluation}
\label{sec:applications}

We compare CF with a robust particle filter (PF) on controlled range-only (RO) and terrain-based (TBN) localization problems that separate outliers, systematic probabilistic model mismatch, association ambiguity, and multimodality caused by terrain ambiguity. These scenarios are representative of difficult nonlinear localization problems in which the feasible state may be large, disconnected, or weakly observable. Simulated data are deliberately used to control each source of difficulty independently and to provide an exact ground truth for assessing enclosure guarantees and statistical success rates. In all cases, the initial information is limited to the bounded search domain specified for each scenario: no localized initial estimate is provided. All quantitative campaigns include outliers, while non-outlier measurements remain consistent with the CF bounded-error model.

\subsection{Common protocol}
\label{sec:applications:protocol}

For each scalar input or measurement, let $\rho$ denote the uncertainty half-width used by the CF. Non-outlier errors are generated with a bias fraction $\alpha$ as
\begin{equation}
	e=\alpha\rho+\eta,\qquad |\eta|\leq(1-\alpha)\rho,
	\label{eq:applications:bias}
\end{equation}
where $\eta$ is a zero-mean Gaussian truncated to the remaining support. Thus,
$|e|\leq\rho$, so the biased non-outlier data remain consistent with the CF
bounded-error model. The bias fraction $\alpha$ is used only to generate the
simulated data and is not provided to either filter. The PF keeps centered
process and measurement models. Both filters use identical simulated data for
each Monte-Carlo seed.

Separately, one outlier is injected every eight observations ($+200$~m
in RO and $50\%$ of the terrain range in TBN). The CF handles these outliers
with $(\ell,q)=(10,2)$. This choice is conservative for the CF, as it allows more
constraint violations than actually occur, so the relaxed intersection retains
a larger set of feasible states and therefore yields weaker contraction.
The PF uses a robust likelihood with a uniform contamination component of weight
$\beta=1/8$ and resamples when the normalized effective sample size (ESS) falls
below $0.5$. Computational workloads span $500$ to $10^5$ PF particles and
$5\times10^{-5}$ to $1$s CF compression budgets, with $100$ runs per operating
point.

A run is considered successful if the correct hypothesis is preserved throughout
the complete trajectory. For the PF, the success indicator is
\begin{equation}
	S_{\rm PF}
	=
	\mathbf 1\!\left[
	\forall k,\
	\min_i
	\|\mathbf p_k^{(i)}-\mathbf p_k^\star\|
	\leq r
	\right],
	\label{eq:applications:success:pf}
\end{equation}
with $r=5$m in RO and $r=3$m in TBN. A CF run is successful if the true
state belongs to the materialized outer paving at every measurement time.
Reliability is then defined as the fraction of successful runs over the
$100$ Monte-Carlo trials. Ground-truth diagnostics are excluded from timing,
and the final paving area $A$ measures the compactness obtained for each
compression budget.

\subsection{Range-only localization}
\label{sec:applications:range}

We consider two range-only scenarios: RO1 uses a single identified beacon, whereas RO4 uses four known but indistinguishable beacons, introducing association ambiguity. For an identified beacon $\mathbf b$, the observation is $\|\mathbf p(t_k)-\mathbf b\|\in[r_k]$. The initial domain is $[-450,250]\times[-350,350]$~m and inputs and ranges arrive every $0.1$ and $10$s, respectively. The heading, speed, and range standard deviations are $0.01$rad, $0.05$m.s$^{-1}$, and $2$m, respectively, with $\pm2\sigma$ CF bounds. RO1 uses $\alpha\in\{0,0.15,0.35,0.50\}$.

For RO1, the distance constraint is handled through the inverse-contractor
construction of Section~\ref{sec:contractors:inverse}, using HC4Revise-style
forward-backward propagation.
In RO4, the source identity is unknown. If $\mathcal C^{\rm dist}_{k,j}$ denotes the distance contractor associated with beacon $\mathbf b_j$, the anonymous observation is expressed directly as a union:
\begin{equation}
 \mathcal C^{\rm RO4}_k=\bigcup_{j=1}^{4}\mathcal C^{\rm dist}_{k,j},
 \qquad
 \mathcal C^{\rm dist}_{k,j}
 =\operatorname{Inv}_{[\|\,\cdot-\mathbf b_j\|]}(\mathcal C_{[r_k]}).
 \label{eq:applications:ro4:union}
\end{equation}
Thus ambiguity is handled directly through the contractor algebra, without introducing an explicit
association state or changing the filter architecture. The PF uses the corresponding equal-prior likelihood mixture.

In unbiased RO1 (Fig.~\ref{fig:applications:states}(a)), both filters retain the correct arc. With $50\%$ bias (Fig.~\ref{fig:applications:states}(b)), the PF drifts away while the CF keeps the trajectory enclosed. RO4 (Fig.~\ref{fig:applications:states}(c)) exhibits a larger initial sampling burden, since four annuli must initially be represented, but later data rapidly resolve the association ambiguity. It is shown qualitatively only.
At $10^5$ particles and $r=5$m (Fig.~\ref{fig:applications:quantitative}(a)), PF success is $99\%$ at zero bias, $61\%$ at $15\%$, and $0\%$ at $35$-$50\%$. Changing $r$ to $2.5$ or $10$m does not alter the strong bias conclusion. 

\begin{figure*}[!t]
\centering
\noindent
\begin{minipage}[t]{0.495\textwidth}\centering
\begin{minipage}[t]{0.485\linewidth}\vspace{0pt}\centering\includegraphics[width=\linewidth]{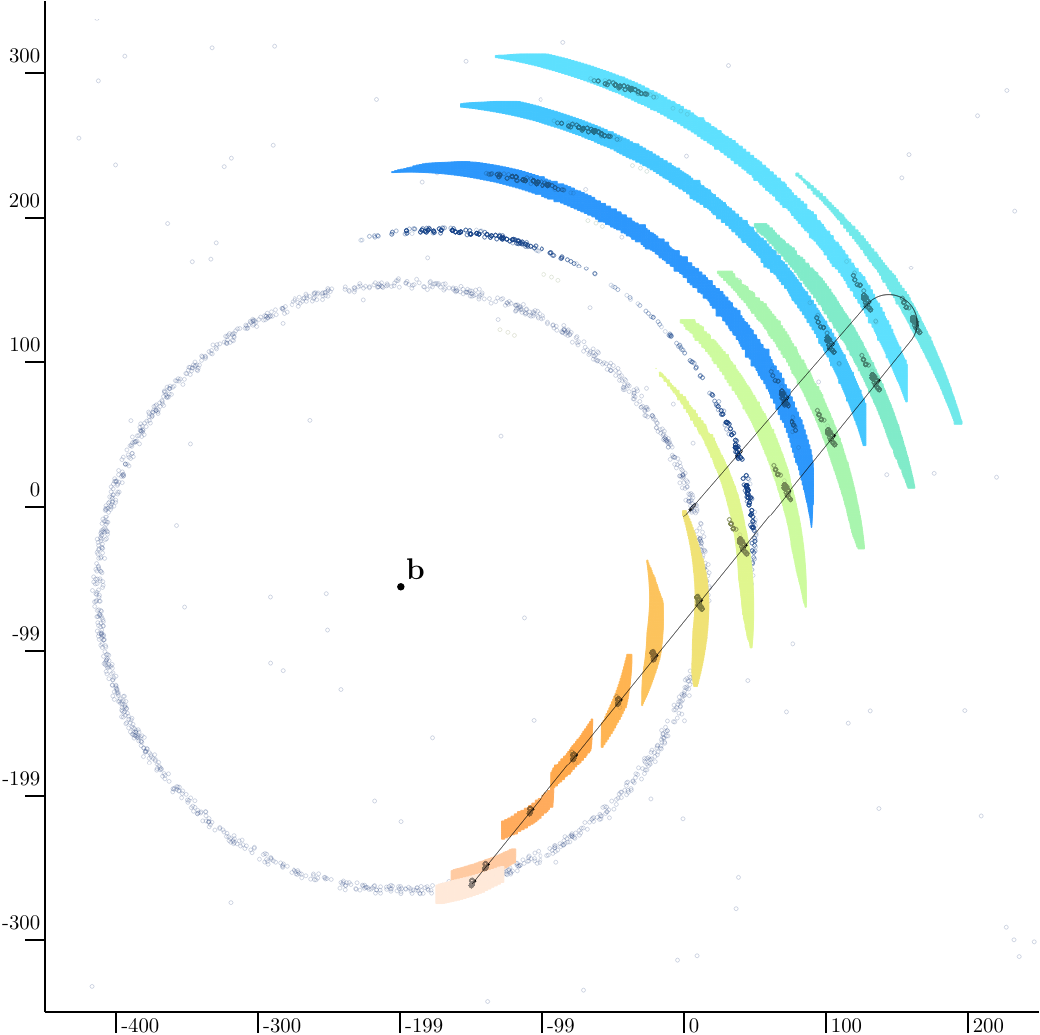}\end{minipage}\hfill
\begin{minipage}[t]{0.485\linewidth}\vspace{0pt}\centering\IfFileExists{./images/ro1_outliers_bias0_zoom.pdf}{\includegraphics[width=\linewidth]{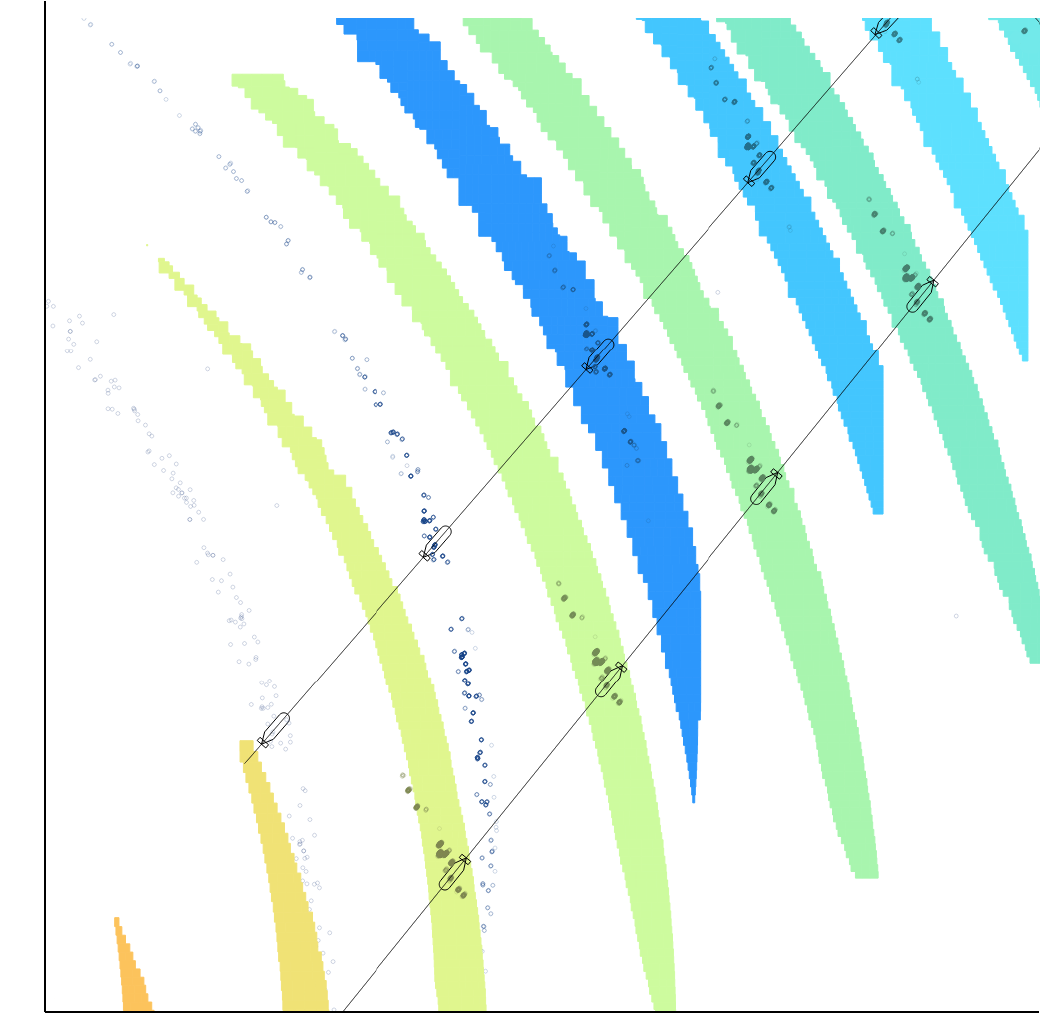}}{\fbox{\parbox[c][\dimexpr\linewidth-2\fboxsep-2\fboxrule\relax][c]{\dimexpr\linewidth-2\fboxsep-2\fboxrule\relax}{\centering\scriptsize zoom}}}\end{minipage}\\[0.1ex]
  {\scriptsize (a) RO1, $0\%$: PF 114.45s, CF 3.51s}
\end{minipage}\hfill
\begin{minipage}[t]{0.495\textwidth}\centering
\begin{minipage}[t]{0.485\linewidth}\vspace{0pt}\centering\includegraphics[width=\linewidth]{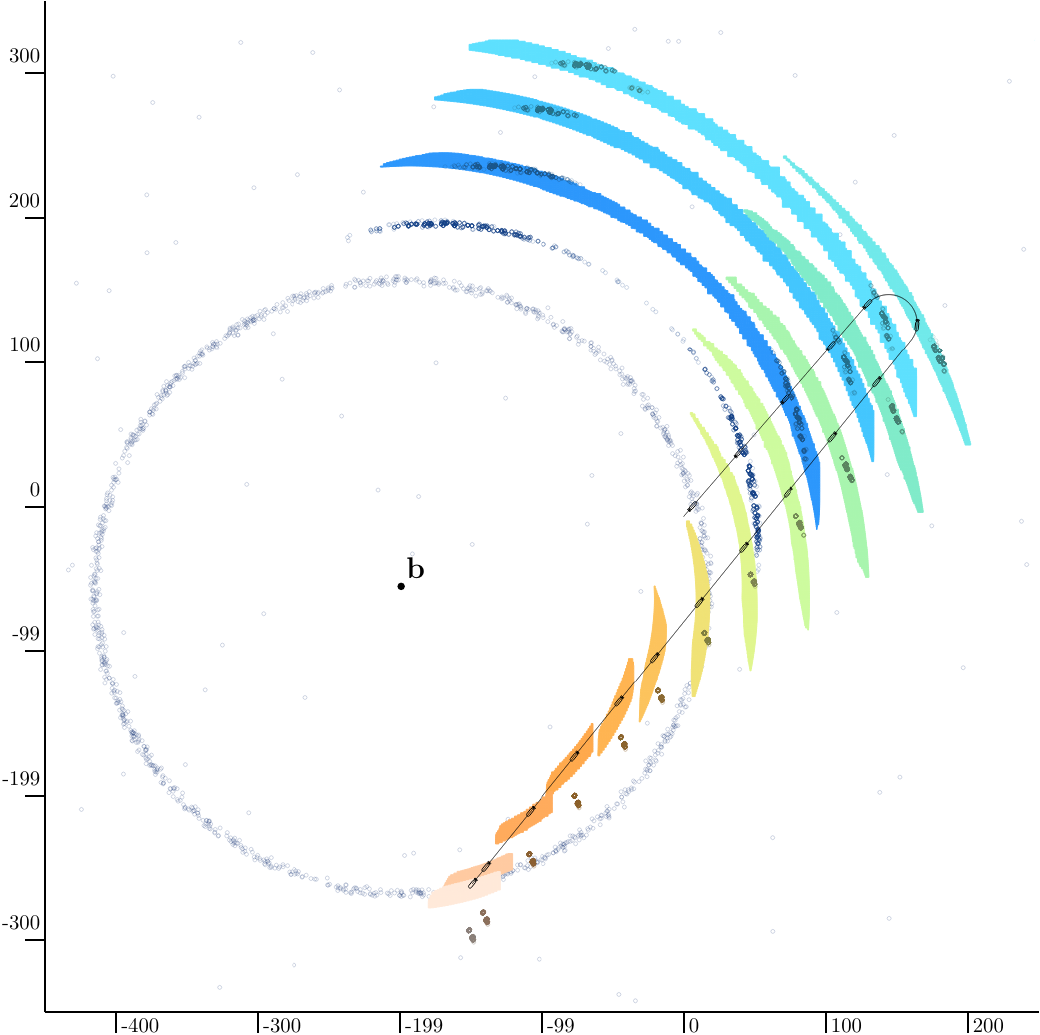}\end{minipage}\hfill
\begin{minipage}[t]{0.485\linewidth}\vspace{0pt}\centering\IfFileExists{./images/ro1_outliers_bias50_zoom.pdf}{\includegraphics[width=\linewidth]{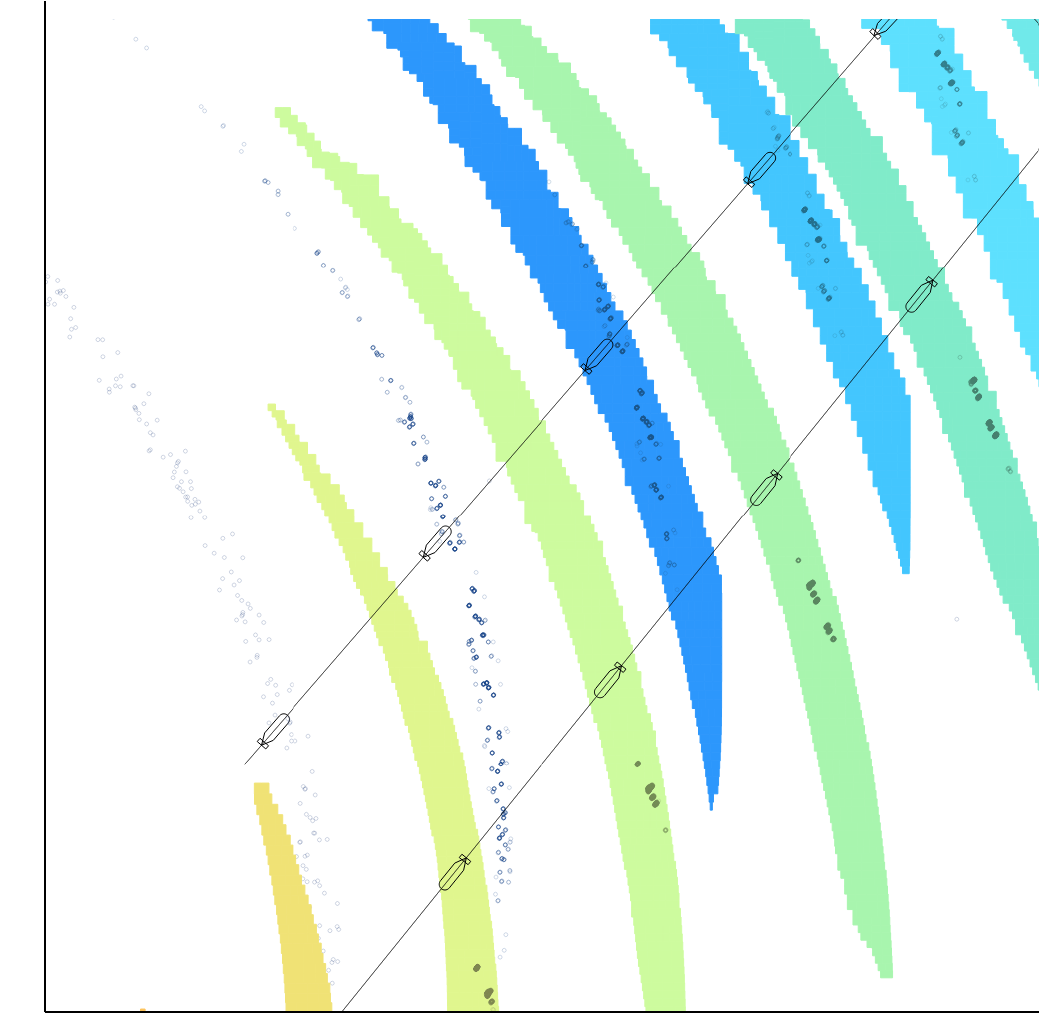}}{\fbox{\parbox[c][\dimexpr\linewidth-2\fboxsep-2\fboxrule\relax][c]{\dimexpr\linewidth-2\fboxsep-2\fboxrule\relax}{\centering\scriptsize zoom}}}\end{minipage}\\[0.1ex]
  {\scriptsize (b) RO1, $50\%$: PF 115.22s, CF 6.84s}
\end{minipage}
\vspace{0.15ex}

\noindent
\begin{minipage}[t]{0.495\textwidth}\centering
\begin{minipage}[t]{0.485\linewidth}\vspace{0pt}\centering\includegraphics[width=\linewidth]{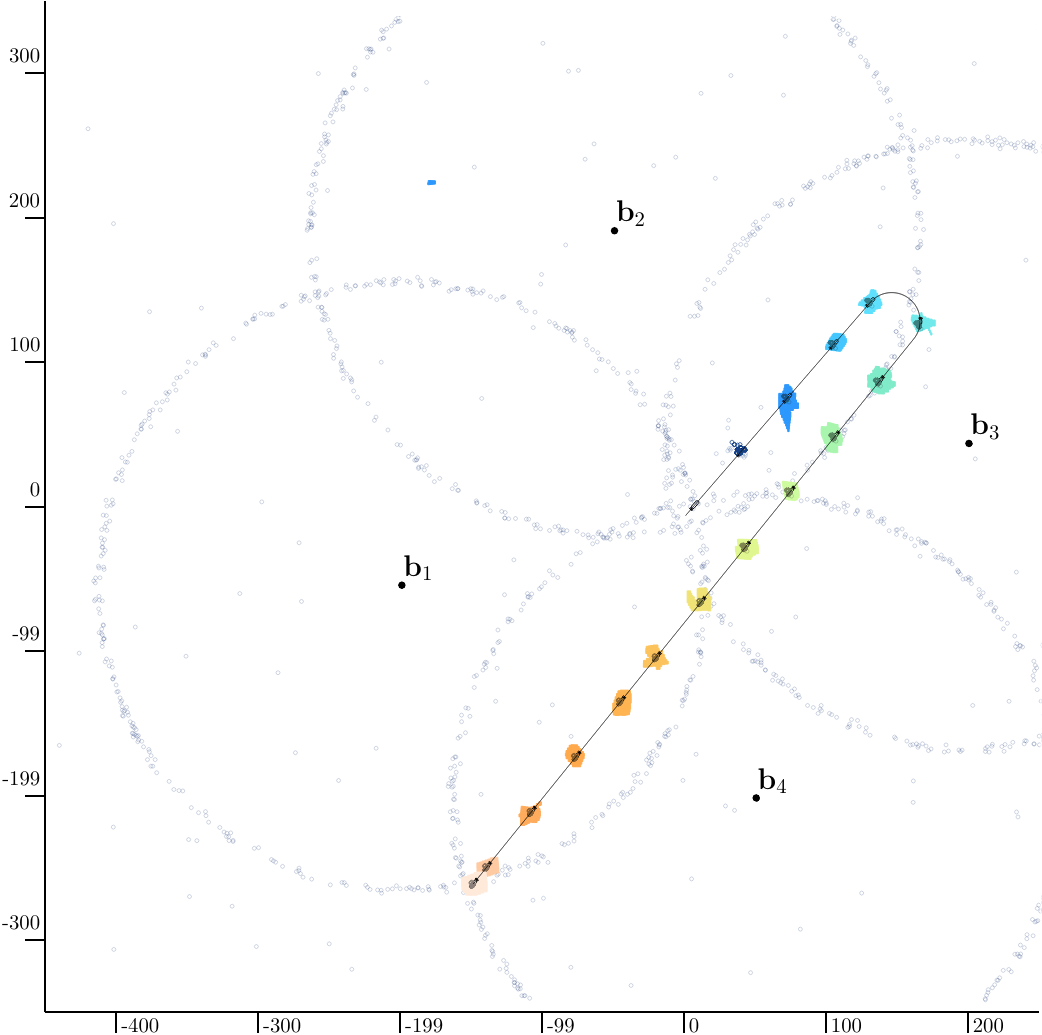}\end{minipage}\hfill
\begin{minipage}[t]{0.485\linewidth}\vspace{0pt}\centering\IfFileExists{./images/ro4_outliers_bias0_zoom.pdf}{\includegraphics[width=\linewidth]{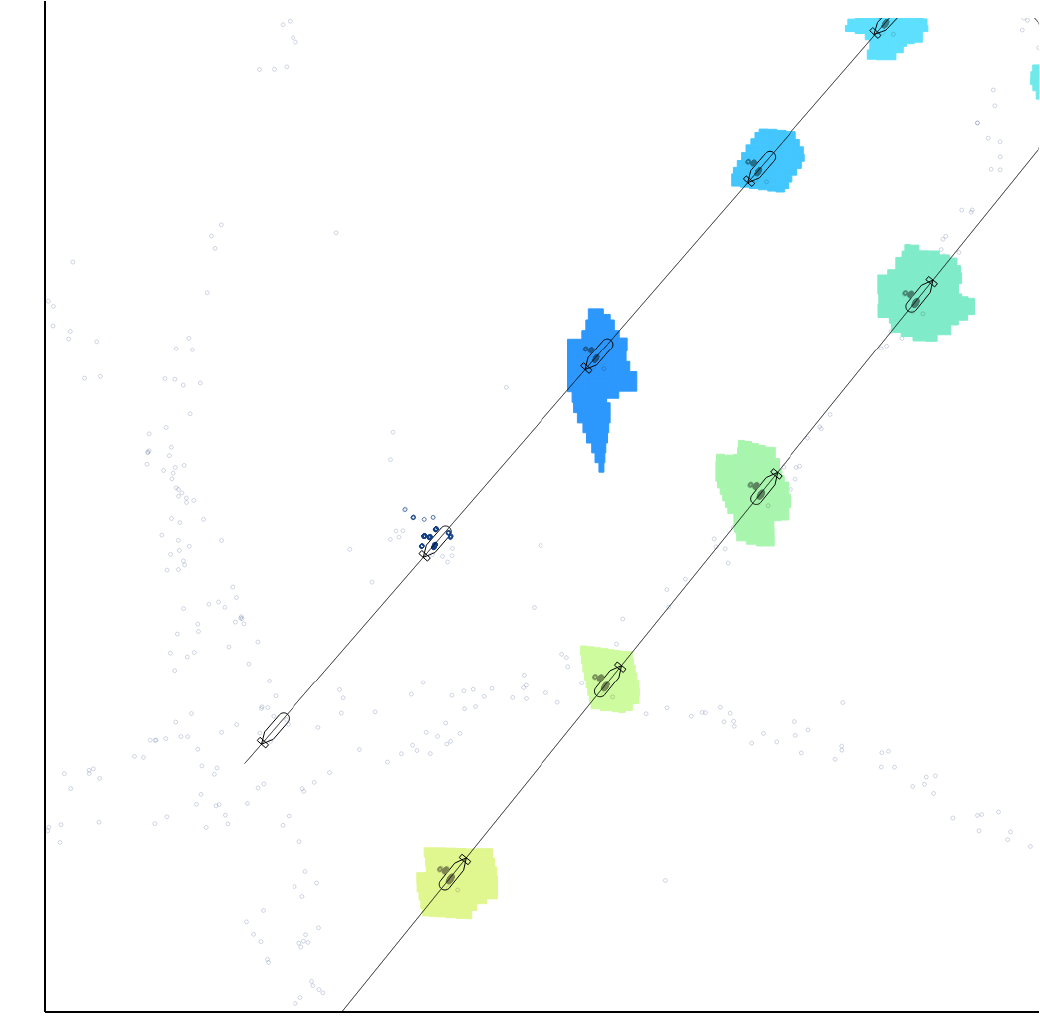}}{\fbox{\parbox[c][\dimexpr\linewidth-2\fboxsep-2\fboxrule\relax][c]{\dimexpr\linewidth-2\fboxsep-2\fboxrule\relax}{\centering\scriptsize zoom}}}\end{minipage}\\[0.1ex]
  {\scriptsize (c) RO4, $0\%$: PF 114.80s, CF 6.82s}
\end{minipage}\hfill
\begin{minipage}[t]{0.495\textwidth}\centering
\begin{minipage}[t]{0.485\linewidth}\vspace{0pt}\centering\includegraphics[width=\linewidth]{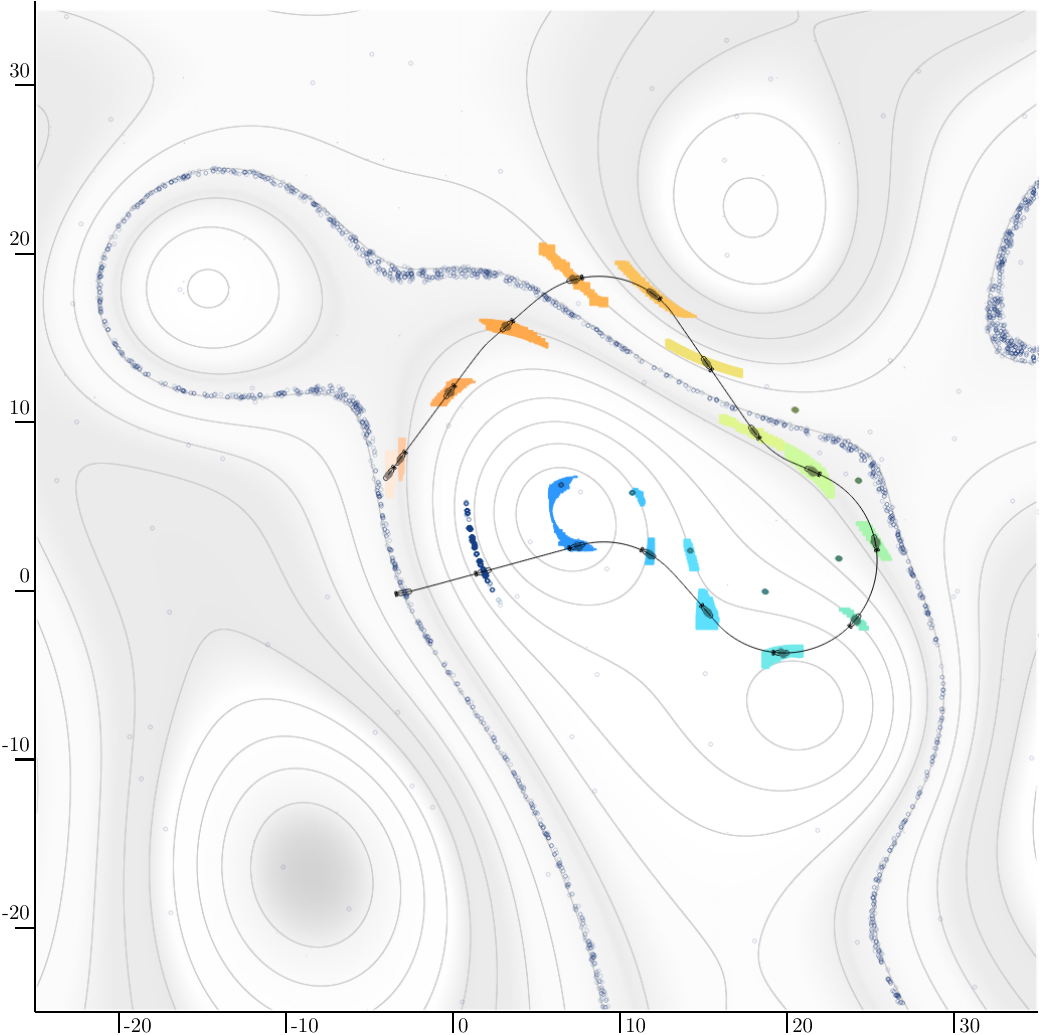}\end{minipage}\hfill
\begin{minipage}[t]{0.485\linewidth}\vspace{0pt}\centering\IfFileExists{./images/tbn1_outliers_bias0_zoom.pdf}{\includegraphics[width=\linewidth]{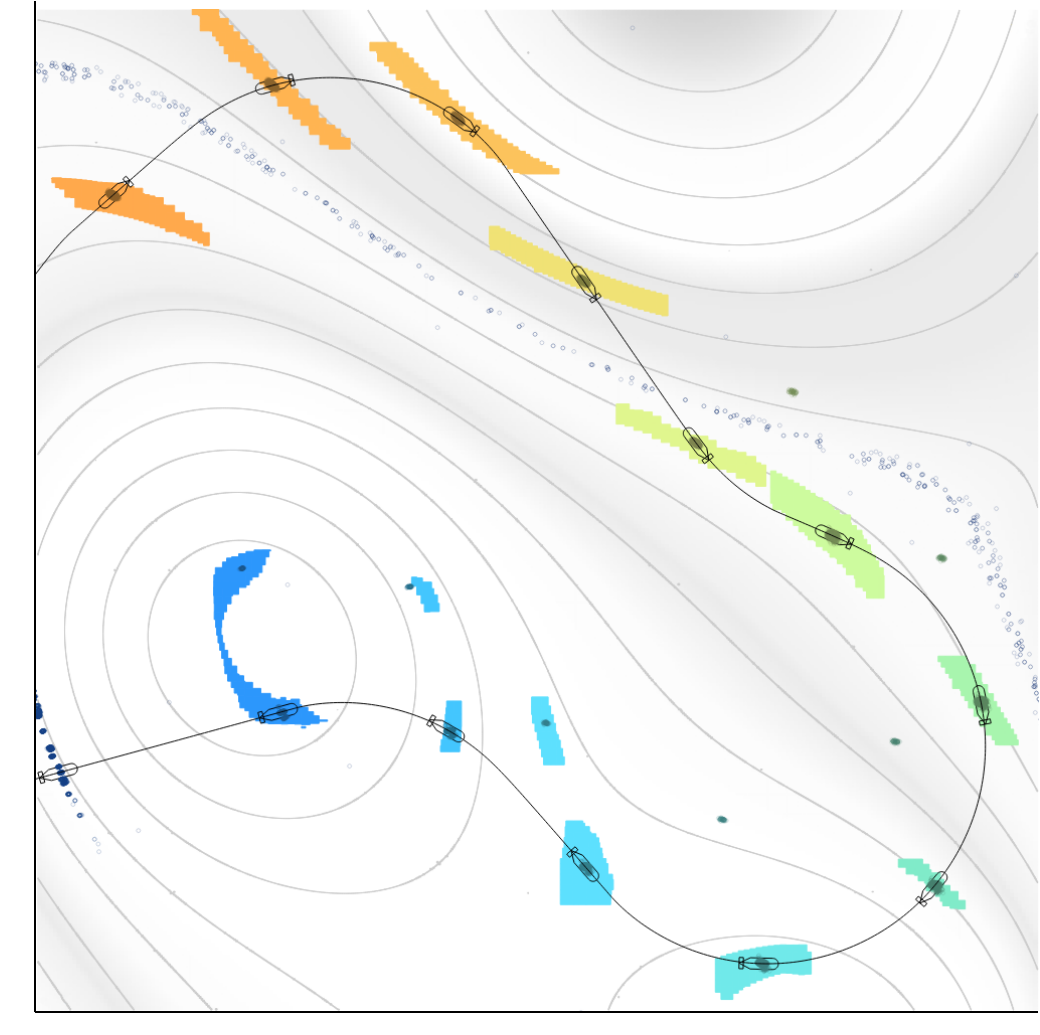}}{\fbox{\parbox[c][\dimexpr\linewidth-2\fboxsep-2\fboxrule\relax][c]{\dimexpr\linewidth-2\fboxsep-2\fboxrule\relax}{\centering\scriptsize zoom}}}\end{minipage}\\[0.1ex]
  {\scriptsize (d) TBN1, $0\%$: PF 13.51s, CF 7.44s}
\end{minipage}
\vspace{0.15ex}

\noindent
\begin{minipage}[t]{0.495\textwidth}\centering
\begin{minipage}[t]{0.485\linewidth}\vspace{0pt}\centering\includegraphics[width=\linewidth]{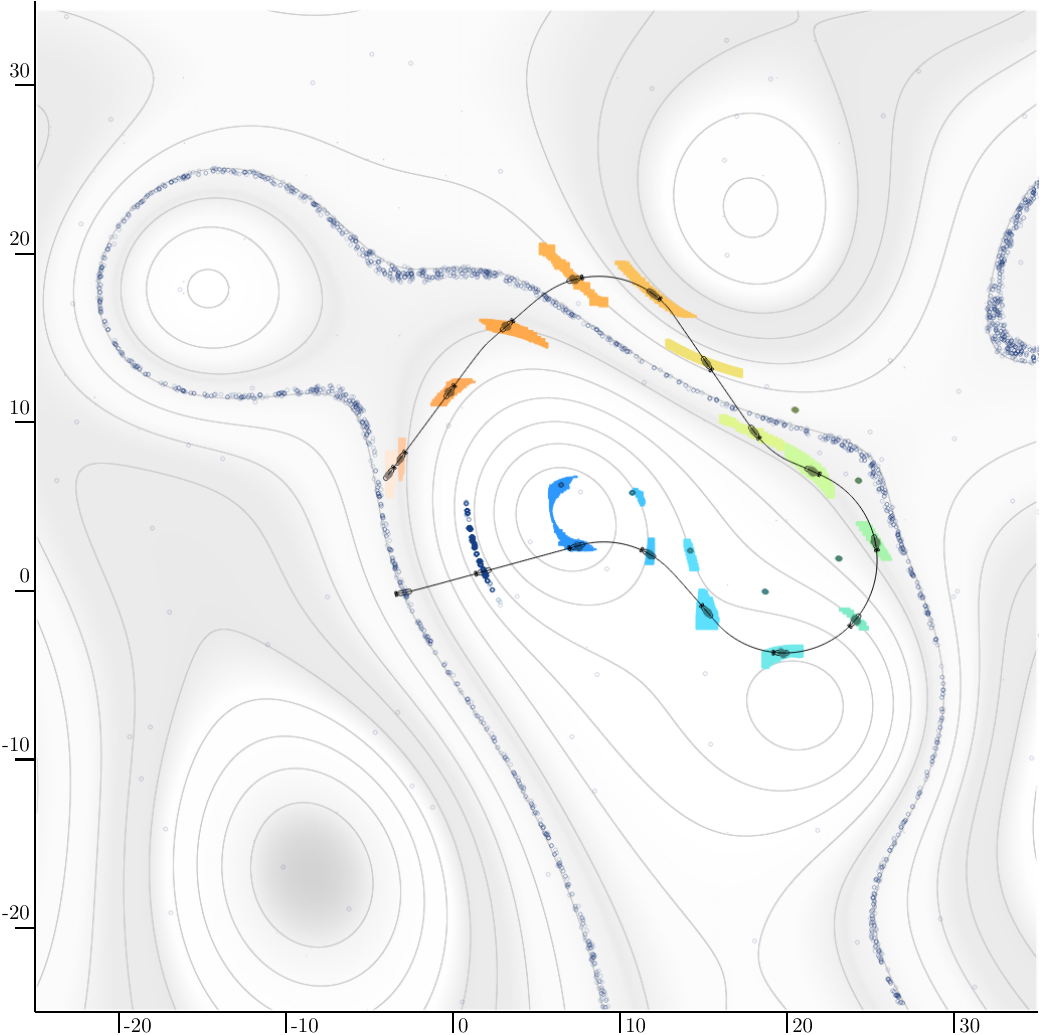}\end{minipage}\hfill
\begin{minipage}[t]{0.485\linewidth}\vspace{0pt}\centering\IfFileExists{./images/tbn1_outliers_bias30_zoom.pdf}{\includegraphics[width=\linewidth]{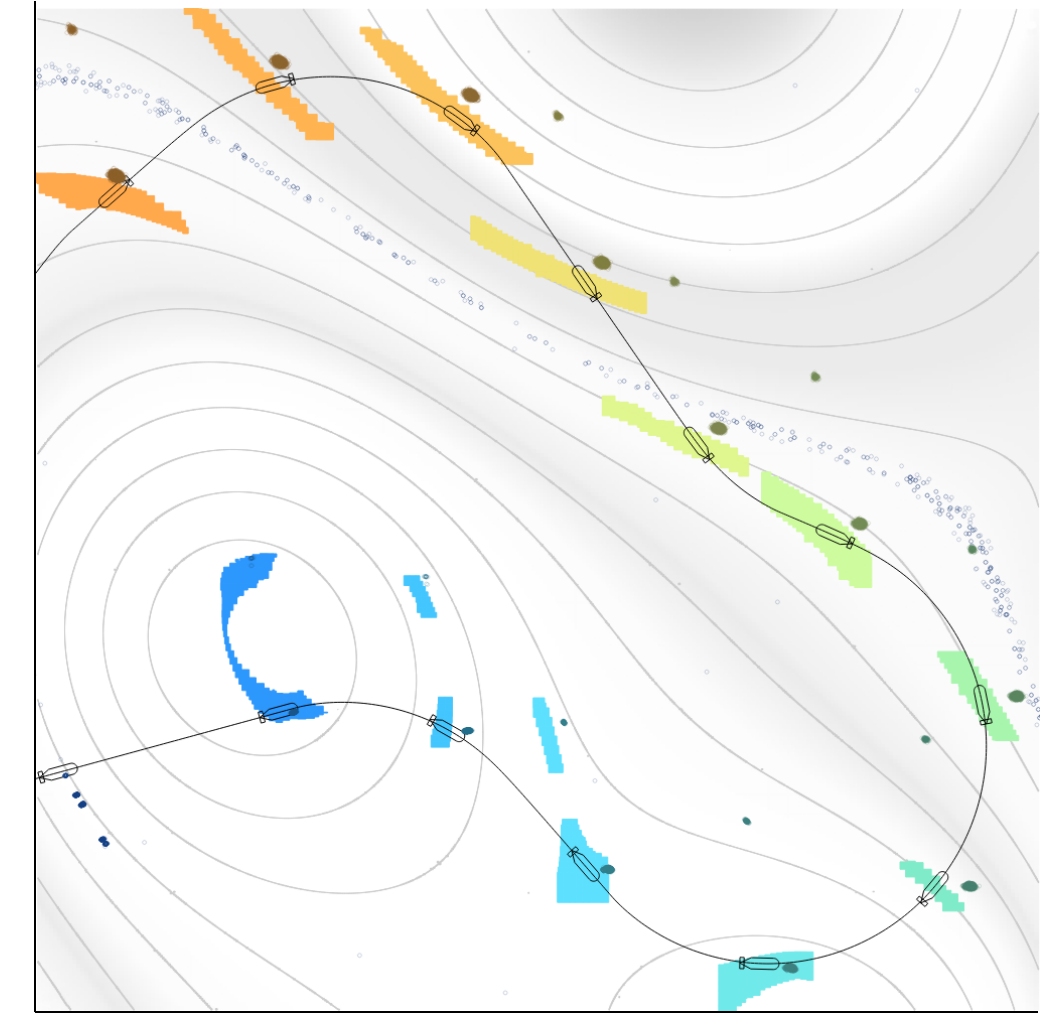}}{\fbox{\parbox[c][\dimexpr\linewidth-2\fboxsep-2\fboxrule\relax][c]{\dimexpr\linewidth-2\fboxsep-2\fboxrule\relax}{\centering\scriptsize zoom}}}\end{minipage}\\[0.1ex]
  {\scriptsize (e) TBN1, $30\%$: PF 13.38s, CF 7.45s}
\end{minipage}\hfill
\begin{minipage}[t]{0.495\textwidth}\centering
\begin{minipage}[t]{0.485\linewidth}\vspace{0pt}\centering\includegraphics[width=\linewidth]{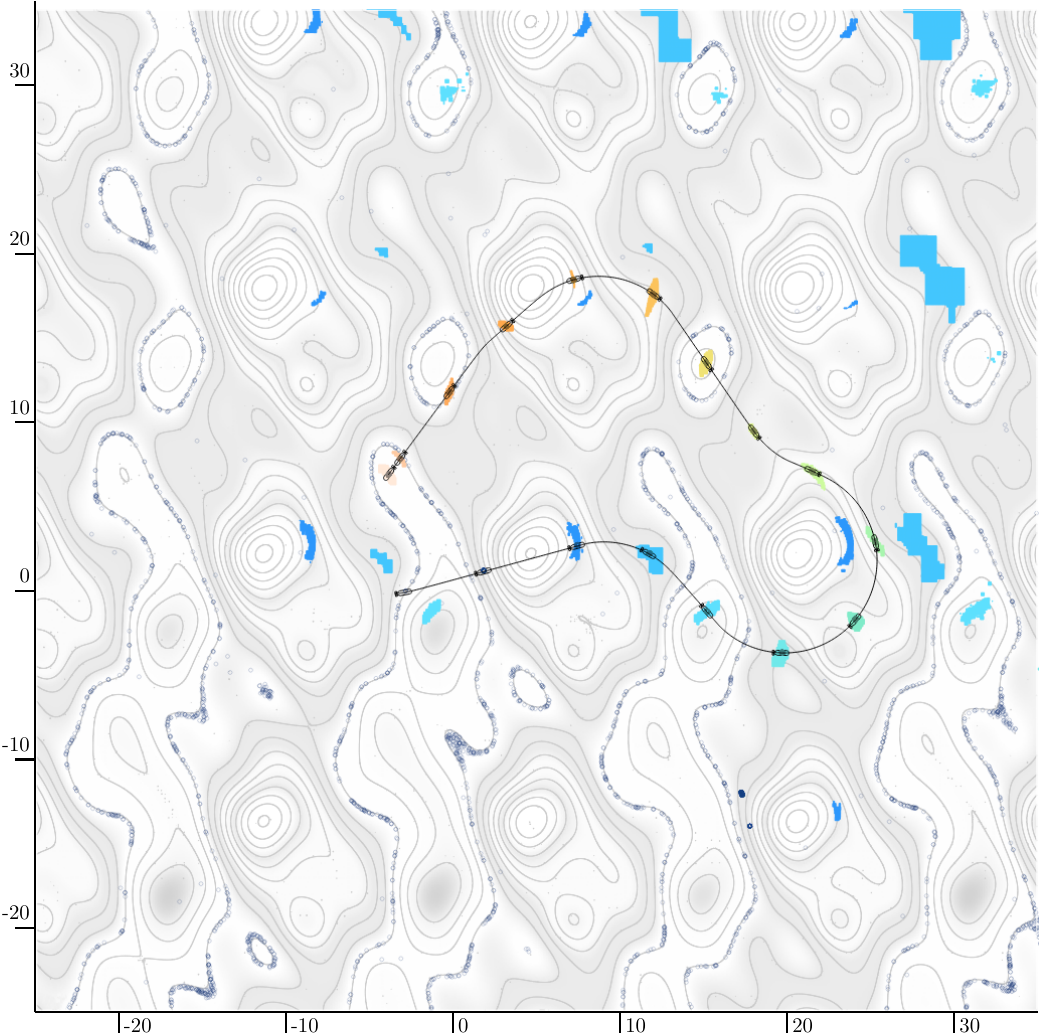}\end{minipage}\hfill
\begin{minipage}[t]{0.485\linewidth}\vspace{0pt}\centering\IfFileExists{./images/tbn2_outliers_bias0_zoom.pdf}{\includegraphics[width=\linewidth]{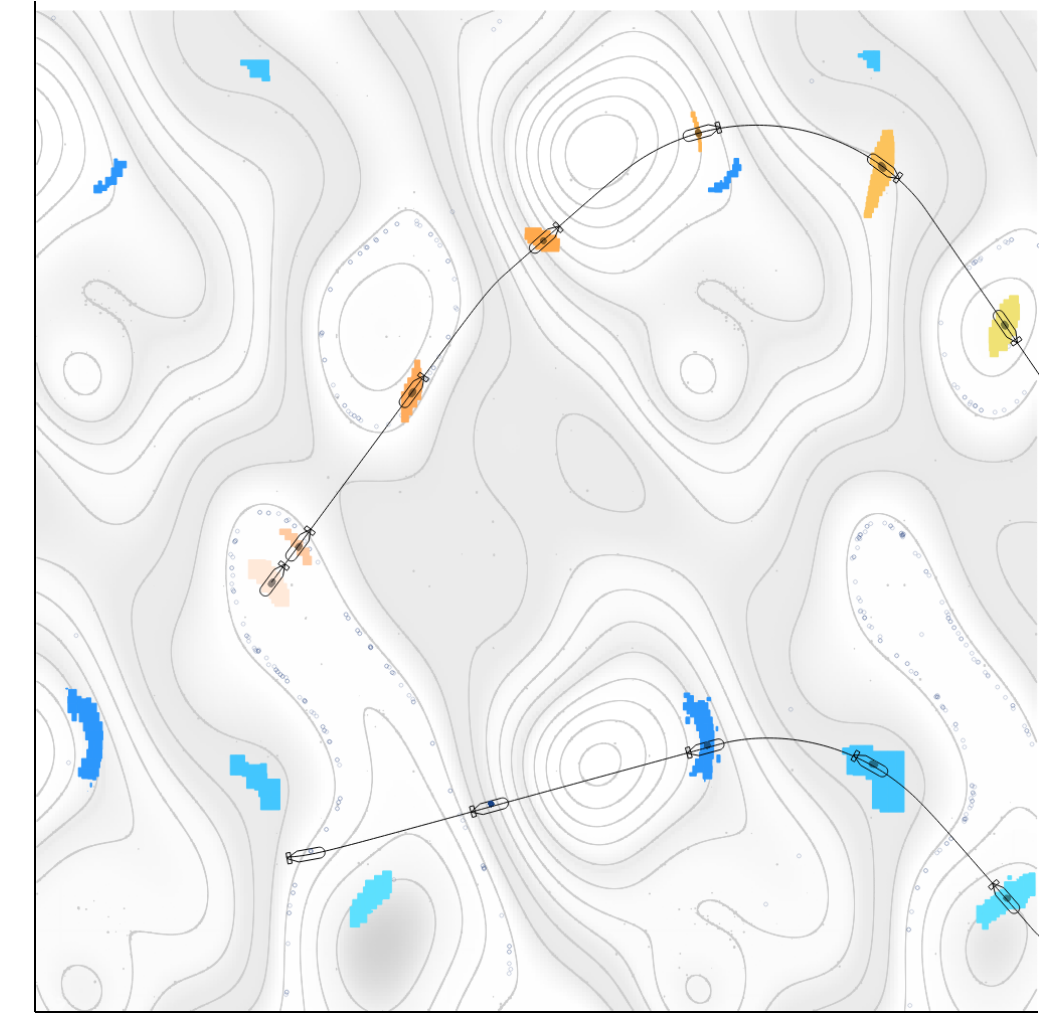}}{\fbox{\parbox[c][\dimexpr\linewidth-2\fboxsep-2\fboxrule\relax][c]{\dimexpr\linewidth-2\fboxsep-2\fboxrule\relax}{\centering\scriptsize zoom}}}\end{minipage}\\[0.1ex]
  {\scriptsize (f) TBN2, $0\%$: PF 30.86s, CF 36.37s}
\end{minipage}
\vspace{-0.4ex}

\caption{Representative CF pavings and PF particles (global view and zoom). Subcaptions
	report measured times. Panels (a)-(c) use $N=10^5$ with CF budgets $0.05$,
	$0.1$, and $0.1$s; (d)-(e) use $N=10^5$, $0.2$s; (f) uses
	$N=2\times10^5$, $1$s. Stars in
	Fig.~\ref{fig:applications:quantitative} mark the benchmarked PF configurations
	of (a), (b), (d), and (e); RO4 is qualitative; in TBN2, only the PF configuration in (f) lies outside the benchmark grid.}
\label{fig:applications:states}
\end{figure*}

\subsection{Terrain-based navigation}
\label{sec:applications:tbn}

The TBN observation is a known scalar field $h:\mathbb R^2\to\mathbb R$,
\begin{equation}
	z_k=h(\mathbf p(t_k))+e_k,\qquad
	h(\mathbf p(t_k))\in[z_k]\quad\text{for the CF}.
	\label{eq:applications:tbn}
\end{equation}
The initial domain is $[-25,35]\times[-25,35]$~m and inputs and terrain observations arrive every $0.1$ and $1$s, respectively. Heading, speed, and terrain standard deviations are $0.01$rad, $0.05$m.s$^{-1}$, and $0.05$, respectively, with $\pm2\sigma$ CF bounds and $\alpha\in\{0,0.10,0.20,0.30\}$.

\begin{figure*}[!t]
\centering
\includegraphics[width=0.24\textwidth]{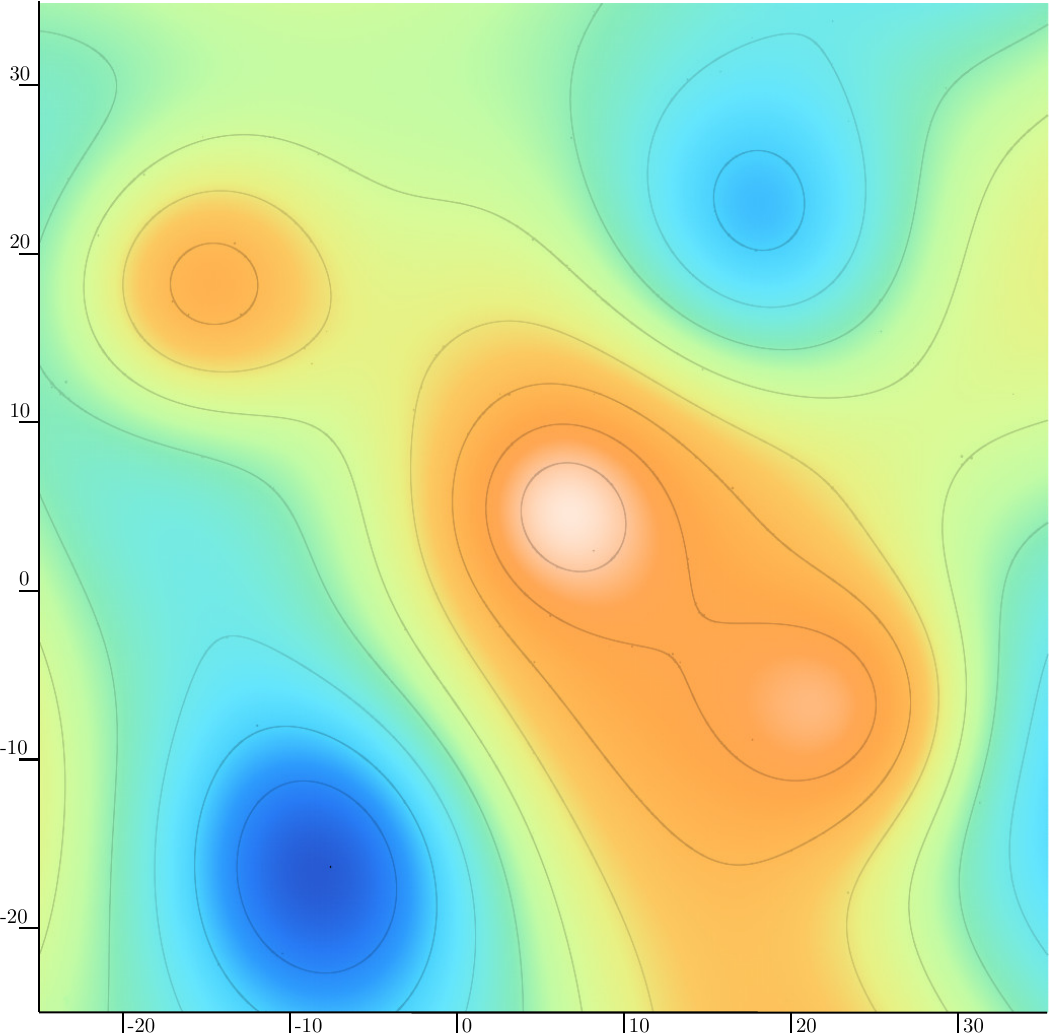}\hspace{1.2em}
\includegraphics[width=0.24\textwidth]{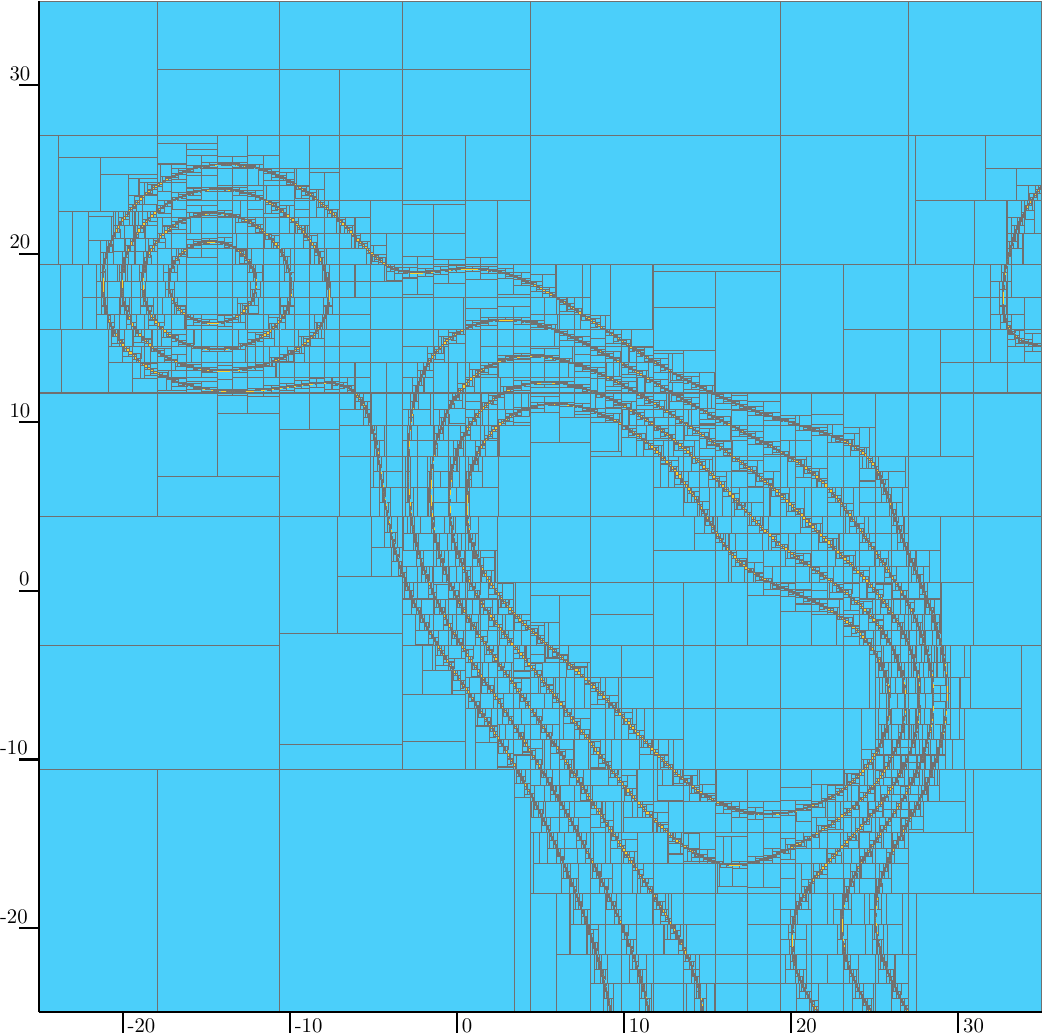}\\[-0.5ex]
{\scriptsize (a) terrain field \hspace{7em} (b) inverse-observation paving}\par
\caption{TBN terrain and direct paving of
	$\textstyle\bigcup_{i=1}^{5}\operatorname{Inv}_{[h]}(\mathcal C_{[z_i]})$
	for five selected observation intervals $[z_i]$. Even a scalar observation can
	induce disconnected, elongated spatial hypotheses before prediction is applied.
	The contraction history visible in the right-hand paving also reveals the
	non-minimality of the terrain contractor on wide boxes: several boxes remain
	only weakly contracted until further subdivision reduces the interval dependency
	effect.}
\label{fig:applications:tbn:contractor}
\end{figure*}

We consider two terrain realizations. TBN1 has sufficiently distinctive relief along the trajectory for spatial alternatives to be rapidly disambiguated, whereas TBN2 contains repeated spatial motifs that sustain multiple plausible localization branches and is deliberately more challenging.

In Fig.~\ref{fig:applications:states}(d)-(e), TBN1 alternatives collapse rapidly. At $30\%$ bias the PF cloud temporarily departs from the trajectory but later reconcentrates, while still satisfying the all-times success criterion. TBN2 combines its repeated spatial motifs with an unfavorable repeated-variable terrain expression. Panel~(f) uses $2\times10^5$ PF particles, outside the benchmark grid, with a benchmarked $1$s CF budget. The PF/CF times are comparable ($30.86/36.37$s).

At $10^5$ particles (Fig.~\ref{fig:applications:quantitative}(b)-(c)), TBN1 PF success is $95/87/82/66\%$ for $0/10/20/30\%$ bias. TBN2 remains harder: doubling from $5\times10^4$ to $10^5$ particles raises success from $39/28/23/21\%$ to $52/43/36/29\%$. Increasing the success radius to $10$m changes little, confirming failures on distinct spatial aliases. CF retains the truth in every TBN run. Its paving area additionally reflects contractor strength.

\begin{figure*}[t]
\centering
\begin{minipage}[t]{0.32\textwidth}\centering
 \includegraphics[height=1.2\linewidth]{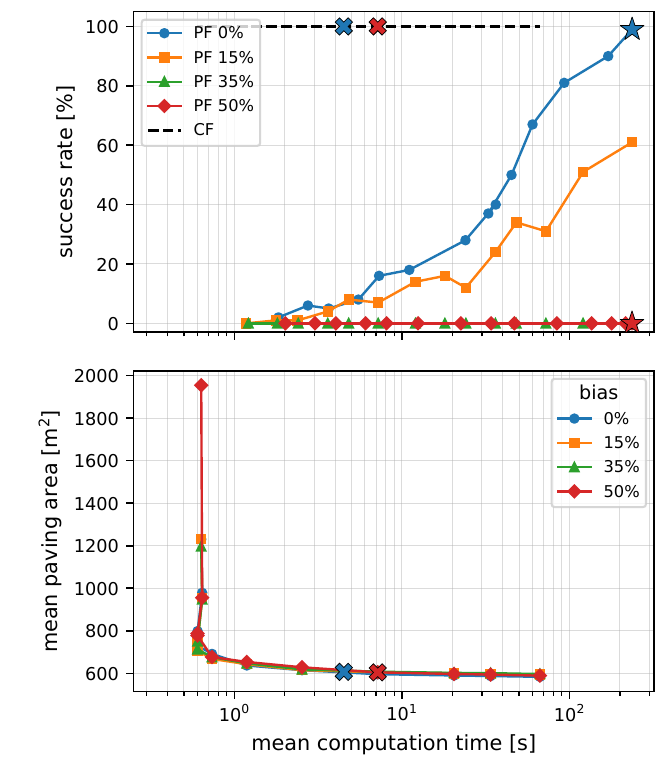}\\[-0.6ex]
 {\scriptsize (a) RO1}
\end{minipage}\hfill
\begin{minipage}[t]{0.32\textwidth}\centering
 \includegraphics[height=1.2\linewidth]{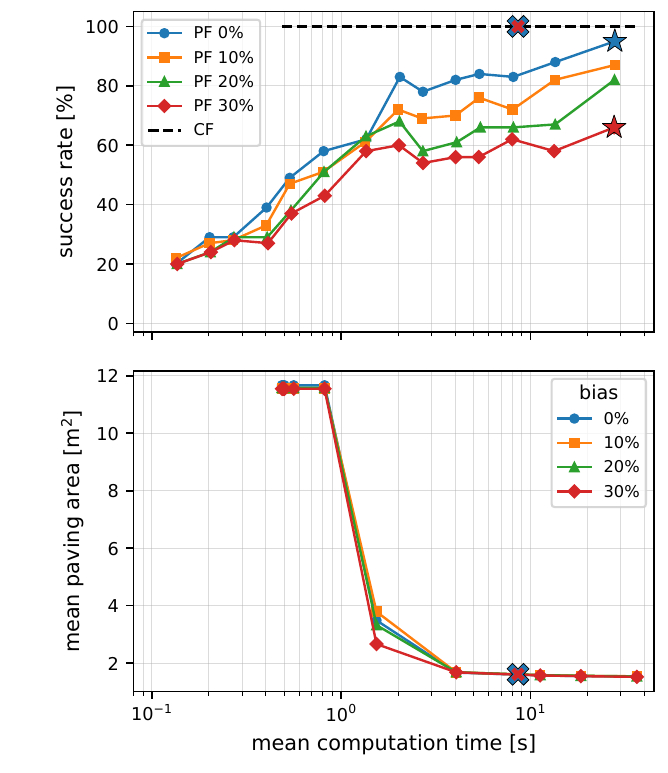}\\[-0.6ex]
 {\scriptsize (b) TBN1}
\end{minipage}\hfill
\begin{minipage}[t]{0.32\textwidth}\centering
 \includegraphics[height=1.2\linewidth]{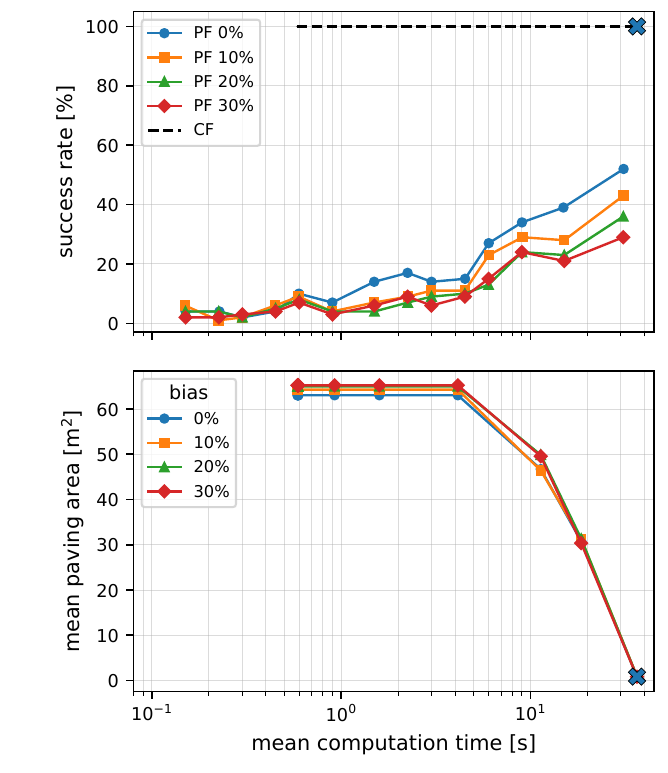}\\[-0.6ex]
 {\scriptsize (c) TBN2}
\end{minipage}
\caption{Top: success rate (reliability) versus measured filtering time; bottom: mean paving area versus measured CF time. $100$ runs per point. PF curves correspond to bias levels; CF is $100\%$ for every tested point and is shown once. Stars identify the $N=10^5$ PF configurations illustrated in Fig.~\ref{fig:applications:states}. The qualitative TBN2 run in Fig.~\ref{fig:applications:states}(f) uses $2\times10^5$ particles and lies outside the plotted benchmark grid.}
\label{fig:applications:quantitative}
\end{figure*}
\FloatBarrier
\subsection{Discussion}
\label{sec:applications:discussion}

The two diagnostics in Fig.~\ref{fig:applications:quantitative} address
complementary aspects of the comparison: reliability measures whether the
correct hypothesis is preserved, whereas paving area measures the compactness
of the CF representation for a given computational effort.

\paragraph{Model mismatch and reliability.}
RO1 isolates persistent probabilistic model mismatch. Increasing the particle
count reduces Monte-Carlo error, but does not correct a systematically
mis-centered probabilistic model. Persistent biases and drifts are a recognized
source of degradation in PFs, and dedicated variants can explicitly estimate or
compensate such errors~\cite{Zhu2017RobustPF}. Here, we deliberately compare the two filters without dedicated bias estimation.
Neither filter uses the bias fraction $\alpha$. The CF only requires
non-outlier errors to remain within their prescribed bounds, which explains why
the truth is retained in every CF run despite the imposed bias. Effects that may
violate these bounds must therefore be accounted for in the uncertainty model.

\paragraph{Contractor strength and paving budget.}
The different area curves can largely be explained by the strength of the
observation contractors. In RO1,
$d_{\mathbf b}(\mathbf p)=\sqrt{(p_1-b_1)^2+(p_2-b_2)^2}$ contains each
position variable only once when the square is treated as a unary operation.
Moore's single-occurrence theorem therefore makes the natural interval
evaluation range-exact, and exact forward-backward propagation yields the
minimal box contraction~\cite{Moore79}.
Accordingly, the RO1 paving is already nearly saturated at a $0.5$s compression
budget, with less than $1\%$ further area reduction at $1$s.

The terrain contractor is less favorable. The analytical function
$h(\mathbf p)$ contains repeated occurrences of $p_1$ and $p_2$, so interval
dependency weakens contraction on wide boxes. HC4Revise-style
forward-backward propagation~\cite{BenhamouICLP99} combined with a centered
form~\cite{Jaulin2024Centered} reduces this effect, but the contractor remains non-minimal away from the
small-box limit; subdivision moves it toward the asymptotically minimal regime. TBN1 still stabilizes rapidly, with only
$0.8$-$1.4\%$ area reduction when the compression budget increases from
$0.5$ to $1$s. TBN2 requires much deeper subdivision: over the same budget
increase, the mean paving area falls from $30$-$31$ to
$0.81$-$0.86$m$^2$. This reduction reflects both improved local contraction
and a finer resolution of the multiple feasible regions. In all cases, the budget
affects paving tightness, not the enclosure guarantee.

\paragraph{Multimodality and representation.}
TBN1 illustrates a case in which later observations resolve the initial spatial
ambiguity, so increasing the number of particles improves PF reliability.
TBN2 is different: repeated terrain motifs sustain distinct spatial aliases, and
additional particles improve sampling without removing this ambiguity at the
tested particle counts. This also highlights a broader issue for classical set-membership
filtering: a single convex set, such as one box, ellipsoid, or zonotope, cannot
represent strongly disconnected hypotheses without enclosing the gaps between
them. Richer representations are therefore needed for such problems, including
unions of sets, multizonotopes, or pavings. Recent terrain-aided
navigation work similarly uses multizonotopes to preserve several spatial
hypotheses under terrain ambiguity~\cite{Ma2025Multizonotope}. CF represents such alternatives symbolically through contractor composition,
without imposing a single convex state representation or requiring explicit
management of multiple convex components.

\paragraph{Cost of validated numerics.}
The reported CF timings include the cost of outward-rounded interval arithmetic, providing guarantees down to floating-point
rounding. This numerical rigor is stronger than what is usually required in robotics,
where the main concern is the algorithmic enclosure of the feasible states. It
is retained here because the implementation relies on validated-numerics tools.
Relaxing these floating-point guarantees could substantially reduce computation
time.
\section{Conclusion}
\label{sec:conclusion}

This paper introduced \emph{Contractor Filters} (CFs), a contractor-valued prediction-update architecture for nonlinear continuous-time state estimation with bounded uncertainties and sporadic outliers. We derived an exact nominal feasible-state recursion and a guaranteed robust state-space enclosure, then implemented prediction and correction through inverse-contractor operations with validated flow inclusions. Conservative sub-paving compression replaces a growing symbolic history by a reusable outer representation without discarding feasible states.
Beyond this enclosure property, the contractor representation provides a
level of abstraction between the filtering equations and their numerical
realization: the same algebraic architecture accommodates nonlinear prediction,
measurement inversion, ambiguity, and outlier handling without prescribing a
specific geometric representation of the state.

The experiments clarify the distinction with particle filtering. Additional particles improve PF reliability when sampling is the dominant limitation, but do not remove strong probabilistic model mismatch or persistent spatial aliases. Across all tested CF operating points, the simulated truth is retained, only the tightness and computational cost of the outer representation vary. TBN2 further shows that obtaining a tight paving can require substantially deeper subdivision, both because the feasible set is strongly multimodal and because the repeated-variable terrain expression introduces interval dependency that weakens contraction on wide boxes. Neither effect compromises the enclosure guarantee.

Future work will target stronger guaranteed contractors for sampled or raster maps, with the aim of reducing interval over-approximation and the subdivision required for a given paving tightness. A backward smoothing pass that propagates future information to past states could extend the framework beyond causal filtering. Temporal uncertainty could also be incorporated into the CF framework, but would require dedicated extensions to represent and invert constraints whose time of occurrence is itself uncertain, possibly relying on tube-based methods. More fundamentally, the CF developed here is a causal estimator: information is propagated forward as measurements arrive. It is therefore not directly suited to global estimation problems such as SLAM, where newly acquired constraints may need to act on past states and on the whole trajectory. Such non-causal problems require richer temporal representations, such as tubes and contractors on tubes \cite{RohouTubeval,robloc}.

\section*{Acknowledgements}

This work was supported by the French National Research Agency (ANR)
through the Electro-Karst (REK) project [grant ANR-22-CE33-0005].

\bibliographystyle{cas-model2-names}
\bibliography{references}
\end{document}